\documentclass[journal]{IEEEtran/IEEEtran}

\IEEEoverridecommandlockouts
\usepackage[noadjust]{cite}
\usepackage{graphics} 
\usepackage{epsfig} 
\usepackage{amsmath} 
\usepackage{amssymb}  
\usepackage{algorithm}
\usepackage{amsthm,array}
\usepackage[noend]{algpseudocode}
\usepackage{color}
\usepackage{multirow}
\usepackage{subcaption}
\usepackage{balance}
\usepackage{booktabs}
\usepackage[dvipsnames]{xcolor}
\colorlet{edit}{ForestGreen}
\colorlet{comment}{blue}

\usepackage{float}

\DeclareMathOperator*{\argmin}{arg\,min}

\usepackage[bookmarks=true]{hyperref}

\newtheorem{definition}{Definition}
\newtheorem{theorem}{Theorem}
\newtheorem{proposition}{Proposition}

\newtheorem{corollary}{Corollary}

\begin{document}

\title{An Ergodic Measure for Active Learning \\ From Equilibrium}

\author{    Ian Abraham,
            Ahalya Prabhakar
            and Todd D. Murphey
    \thanks{Ian Abraham, Ahalya Prabhakar, and Todd D. Murphey are all with the Department of
    Mechanical Engineering at Northwestern University, Evanston, IL, 60208. corresponding e-mail:
    see i-abr@u.northwestern.edu }
    \thanks{For videos of each example and demo code please visit
    (\url{https://sites.google.com/view/kle3/home})}
}

\markboth{Journal of \LaTeX\ Class Files,~Vol.~14, No.~8, August~2015}%
{Shell \MakeLowercase{\textit{et al.}}: Bare Demo of IEEEtran.cls for IEEE Journals}

\IEEEspecialpapernotice{(Invited Paper)}

\maketitle

\begin{abstract}

    This paper develops KL-Ergodic Exploration from Equilibrium ($\text{KL-E}^3$), a method for
    robotic systems to integrate stability into actively generating informative measurements through
    ergodic exploration. Ergodic exploration enables robotic systems to indirectly
    sample from informative spatial distributions globally, avoiding local optima, and without the
    need to evaluate the derivatives of the distribution against the robot dynamics. Using hybrid
    systems theory, we derive a controller that allows a robot to exploit equilibrium policies
    (i.e., policies that solve a task) while allowing the robot to explore and generate informative
    data using an ergodic measure that can extend to high-dimensional states. We show
    that our method is able to maintain Lyapunov attractiveness with respect to the equilibrium task
    while actively generating data for learning tasks such, as Bayesian optimization, model learning,
    and off-policy reinforcement learning. In each example, we show that our proposed method
    is capable of generating an informative distribution of data while synthesizing smooth control
    signals. We illustrate these examples using simulated systems and provide
    simplification of our method for real-time online learning in robotic systems.

\end{abstract}

Note to Practitioners:
\begin{abstract}

    Robotic systems need to adapt to sensor measurements and learn to exploit an understanding of the world around them
    such that they can truly begin to experiment in the real world. Standard learning methods do not have any
    restrictions on how the robot can explore and learn, making the robot dynamically volatile. Those that do, are often
    too restrictive in terms of the stability of the robot, resulting in a lack of improved learning due to poor data
    collection. Applying our method would allow robotic systems to be able to adapt online without the need for human
    intervention. We show that taking into account both the dynamics of the robot and the statistics of where the robot
    has been, we are able to naturally encode where the robot needs to explore and collect measurements for efficient
    learning that is dynamically safe. With our method we are able to effectively learn while being energetically
    efficient compared to state-of-the-art active learning methods. Our approach accomplishes such tasks in a single
    execution of the robotic system, i.e., the robot does not need human intervention to reset it. Future work will
    consider multi-agent robotic systems that actively learn and explore in a team of collaborative robots.

\end{abstract}

\begin{IEEEkeywords}
    Active Learning, Active Exploration, Online Learning, Stable Learning
\end{IEEEkeywords}

\IEEEpeerreviewmaketitle

\section{Introduction} \label{sec:introduction}

    \IEEEPARstart{R}{obot} learning has proven to be a challenge for real-world application. This is
    partially due to the ineffectiveness of passive data acquisition for learning and the necessity
    for data-driven actions for collecting informative data. What makes this problem even more
    difficult is that generating data for robotic systems is often an unstable process. It
    involves generating measurements dependent upon physical motion of the robot. As a result, safe
    data collection through exploration becomes a challenge. The problem
    becomes exacerbated when memory and constraints (i.e., data must come from a single
    roll out) are imposed on the robot. Thus, robotic systems need to adapt subject to data in a
    systematic and informative manner while preserving a notion of equilibrium as a means of safety
    for itself, the environment, and potentially humans. In this paper, we address these issues by
    developing an algorithm that draws on hybrid systems theory~\cite{axelsson_JOTA_modeinsertion}
    and ergodic exploration~\cite{miller2016ergodic, mavrommati2018real, ayvali2017ergodic,
    abraham2017ergodic,abraham2018data, abraham2018decentralized}. Our approach enables robots to
    generate and collect informative data while guaranteeing Lyapunov
    attractiveness~\cite{polyakov2014stability} with respect to an equilibrium task.\footnote{Often
    such equilibrium tasks can be thought of as stabilization, but can be viewed as running or
    executing a learned skill consistently like swinging up and stabilizing a cart pole.}

    Actively collecting data and learning are often characterized as being part of the same problem of learning
    from experience~\cite{kormushev_robotmotorskills_em_rl, reinhart_AuRo_skill_babble}. This is
    generally seen in the field of reinforcement learning (RL) where attempts at a task, as well as
    learning from the outcome of actions, are used to learn policies and predictive
    models~\cite{kormushev_robotmotorskills_em_rl, mckinnon_multimodal_gp_learning_online}. Much of
    the work in the field is dedicated towards generating a sufficiently large distribution of data
    such that it encompasses unforeseen events, allowing generalization to real-world
    application~\cite{mckinnon_multimodal_gp_learning_online, kormushev_robotmotorskills_em_rl,
    tan_RSS_sim_to_real, marco2017virtual} however inefficient the method. In this work, rather than
    trying to generate a large distribution of data given many attempts at a task, we
    seek to generate an informative data distribution, and learn from the collected
    distribution of data, as two separate problems, where we focus on actively and intelligently
    collecting data in a manner that is safe and efficient while adopting existing methods that
    enable learning.

    Current safe learning methods typically provide some bound on the worst outcome model using
    probabilistic approaches~\cite{berkenkamp2017safe}, but often only consider the safety with
    respect to the task and not with respect to the exploration process. We focus on problems where
    exploring for data intersects with exploring the physical space of robots such that actions that
    are capable of generating the best set of data can destabilize the robot.

    In this work we treat active exploration for data as an ergodic exploration problem, where time
    spent during the trajectory of the robot is proportional to the measure of informative data in that region. As a
    result, we are able to efficiently use the physical motion of the robot by focusing on the dynamic area coverage in
    the search space (as opposed to directly generating samples from the most informative regions). With this approach,
    we are able to integrate known equilibria and feedback policies (we refer to these as equilibrium policies) which
    provide attractiveness guarantees\textemdash that the robot will eventually return to an equilibrium \textemdash
    while providing the control authority that allows the robot to actively seek out and collect informative data in
    order to later solve a learning task. Our contributions are summarized as follows:

    \begin{itemize}
        \item Developing a method which extends ergodic exploration to higher dimensional state-spaces using a sample-based measure.
        \item Synthesis of a control signal that exploits known equilibrium policies.
        \item Present theoretical results on the Lyapunov attractiveness centered around equilibrium policies of our method.
        \item Illustrate our method for improving the sample efficiency and quality of data collected for example learning goals.
    \end{itemize}

    We structure the paper as follows: Section~\ref{sec:related_work} provides a list of related
    work, Section~\ref{sec:eee} defines the problem statement for this work. Section~\ref{sec:kleee}
    introduces our approximation to an ergodic metric and Section~\ref{sec:algo} formulates the
    ergodic exploration algorithm for active learning and exploration from equilibrium.
    Section~\ref{sec:ex} provides examples where our method is applicable. Last,
    Section~\ref{sec:conc} provides concluding remarks on our method and future directions.

\section{Related Work} \label{sec:related_work}

    \textbf{Active Exploration:}
    Existing work generally formulates problems of active exploration
    as information maximization with respect to a known parameterized model~\cite{lin2017direct,
    bourgault2002information}. The problem with this approach is the abundance of local
    optima~\cite{miller2016ergodic, bourgault2002information} which the robotic system needs to
    overcome, resulting in insufficient data collection. Other approaches have sought to solve this
    problem by viewing information maximization as an area coverage problem~\cite{miller2016ergodic,
    ayvali2017ergodic}. Ergodic exploration, in particular, has remedied the issue of local optima
    by using the ergodic metric to minimize the Sobelov distance~\cite{arnold1992sobolev} from the
    time-averaged statistics of the robot's trajectory to the expected information in the explored
    region. This enables both exploration (quickly in low information regions) and exploitation
    (spending significant amount of time in highly informative regions) in order to avoid local
    optima and harvest informative measurements. Our work utilizes this concept of ergodicity to
    improve how robotic systems explore and learn from data.

    \textbf{Ergodicity and Ergodic Exploration:}
    A downside with the current methods for generating ergodic exploration in robots is that they
    assumes that the model of the robot is fully known. Moreover, there is little guarantee that the
    robot will not destabilize during the exploration process. This becomes an issue when the robot
    must explore part of its own state-space (i.e., velocity space) in order to generate informative
    data. Another issue is that these methods do not scale well with the dimensionality of the
    search space, making experimental applications with this approach challenging due to
    computational limitations. Our approach overcomes these issues by using a sample-based
    KL-divergence measure~\cite{ayvali2017ergodic} as a replacement for the ergodic metric. This
    form of measure has been used previously; however, it relied on motion primitives in order to
    compute control actions~\cite{ayvali2017ergodic}. We show that we can generate a continuous
    control signal that minimizes this ergodic measure using hybrid systems theory. The
    same approach is then shown to be readily amenable to existing equilibrium policies. As a
    result, we can use approximate models of dynamical systems instead of complete dynamic
    reconstructions in order to actively generate data while ensuring safety in the exploration
    process through a notion of attractiveness.

    \textbf{Off-Policy Learning:}
    In this work, we utilize concepts from off-policy learning methods~\cite{precup2001off,
    lillicrap2015continuous} which are a set of methods that divides the learning and data
    generation phases. With these methods, data is first generated using some random policy and a
    value function is learned from rewards gathered. The value function is then used to create a
    policy which then updates the value function~\cite{kaelbling1996reinforcement}. Generating more
    data does not require directly using the policy; however, the most common practice is to use the
    learned policy with added noise to guide the policy learning. These methods often rely on
    samples from a buffer of prior data during the training process rather than learning directly
    through the application of the policy. As such, they are more sample-efficient and can
    reuse existing data. A disadvantage with off-policy methods is that they are highly dependent on
    the distribution of data, resulting in an often unstable learning process. Our approach focuses
    on improving the data generation process through ergodic exploration to generate an informed
    distribution of data. As a result, the learning process is able to attain improved results from
    the generated distribution of data and retain its sample-efficiency.

    \textbf{Bayesian Optimization:}
    Our work is most related to the structure of Bayesian optimization~\cite{frazier2018tutorial,
    snoek2012practical, calandra2016bayesian}. In Bayesian optimization, the goal is to find the
    maximum of an objective function which is unknown. At each iteration of Bayesian optimization,
    the unknown objective is sampled and a probabilistic model is generated. An acquisition function
    is then used as a metric for an ``active learner'' to find the next best sample. This loop
    repeats until a maximum is found. In our work, the ``active learner'' is the robot itself which
    must abide by the physics that governs its motion. As a result, the assumption that the active
    learner has the ability to sample anywhere in the search space is lost. Another difference is
    that instead of using a sample-based method to find the subsequent sampling position, as done in
    Bayesian optimization, we use ergodic exploration to generate a set of samples proportional to
    some spatial distribution. We pose the spatial distribution as an acquisition function which we
    show in Section~\ref{sec:bayes_opt}. Thus, the active learner is able to sample from regions which have
    lower probability densities quickly and spending time in regions which are likely to produce an
    optima. Note that it is possible for one to directly use the derivatives of the acquisition
    function and a model of the robot's dynamics to search for the best actions that a robot can
    take to sample from the objective; however, it is likely that similar issues with information
    maximization and local optima will occur.

    \textbf{Information Maximization:}
    Last, we review work that addresses the data-inefficiency problem through information
    maximization~\cite{schwager_robotics_inf_gather}. These methods work by either direct
    maximization of an information measure or by pruning a data-set based on some information
    measure~\cite{nguyen_NEURO_incremental_sparse_gp}. These methods still suffer from the problem
    of local minima due to a lack of exploration or non-convex information
    objectives~\cite{ucinski_CRC_optimal_meas}. Our work uses ergodicity as a way to gauge how much
    a robot should be sampling from the exploration space. As a result, the motion of the robot by
    minimizing an ergodic measure will automatically optimize where and for how long the robot
    should sample from, avoiding the need to prune from a data-set and sample from multiple highly
    informative peaks.

    The following section introduces ergodicity and ergodic exploration and defines the problem
    statement for this work.

\section{Preliminaries: Ergodicity and The Ergodic Metric for Exploration and Exploitation} \label{sec:eee}

    This section serves as a preliminary section that introduces conceptual background
    important to the rest of the paper. We first motivate ergodicity as an approach to the exploration vs. exploitation
    problem and the need for an ergodic measure. Then, we introduce ergodicity and ergodic exploration as the resulting
    outcome of optimizing an ergodic metric.

    The exploration vs. exploitation problem is a problem in robot learning where the robot must
    deal with the choosing to exploit what it already knows or explore for more information, which
    entails running the risk of damaging itself or collecting bad data. Ergodic exploration treats
    the problem of exploration and exploitation as a problem of matching a spatial distributions to
    a time-averaged distribution\textemdash that is, the probability of a positive outcome given a state is
    directly related to the time spent at that state. Thus, more time is spent in regions where
    there are positive outcomes (exploitation) and quickly explores states where there is low
    probability of a positive outcome (exploration).

    \begin{definition} \label{def:erg}
        Ergodicity, in robotics, is defined when a robot whose time-averaged statistics
        over its states is equivalent to the spatial statistics of an arbitrarily defined
        target distribution that intersects those states.
    \end{definition}

    The exact specifications of a spatial statistic varies depending on the underlying task and is
    defined for different tasks in Sections~\ref{sec:ex}. For now, let us define the time-averaged
    statistics of a robot by considering its trajectory $x(t):\mathbb{R} \to \mathbb{R}^n$
    $\forall t \in \left[t_0, t_f\right]$ generated through an arbitrary control process
    $u(t) : \mathbb{R} \to \mathbb{R}^m$.

    \begin{definition}\label{def:fourier_recon}
        Given a search domain $\mathcal{S}^v \subset \mathbb{R}^{n+m}$ where $v \le n + m$, the
        time-averaged statistics (i.e., the time the robot spends in regions of the search domain
        $\mathcal{S}^v$) of a trajectory $x(t)$ is defined as
        \begin{equation}
            c(s \mid x(t)) = \frac{1}{t_f - t_0}\int_{t_0}^{t_f} \delta \left[ s - \bar{x}(t) \right]dt,
        \end{equation}
        where $s \in \mathcal{S}^v$ is a point in the search domain, $\bar{x}(t)$ is the component
        of the robot's state $x(t)$ that intersects the search domain, and $\delta[\cdot]$ is the
        Dirac delta function.
    \end{definition}
    In general, the target spatial statistics are defined through its probability density
    function $p(s)$ where $p(s) > 0$, and $\int_{\mathcal{S}^v} p(s) ds = 1$. Given a target spatial
    distribution $p(s)$, we can calculate an ergodic metric as the distance between the Fourier
    decomposition of $p(s)$ and $c(s \mid x(t))$: \footnote{This distance is known as the
    Sobelov distance.}

    \begin{equation}
        \mathcal{E}(x(t)) = \sum_{k\in\mathbb{N}^v} \Lambda_k \left(c_k - p_k \right)^2
    \end{equation}
    where $\Lambda_k$ is a weight on the harmonics defined in~\cite{mathew2011metrics},
    \begin{align}
        c_k & = \frac{1}{t_f-t_0}\int_{t_0}^{t_f} F_k(x(t)) dt, \nonumber \\
        p_k & = \int_{\mathcal{S}^v} p(s) F_k(s) ds,  \nonumber
    \end{align}
    and $F_k(s)$ is the $k^\text{th}$ Fourier basis function. Minimizing the ergodic metric results in
    the time-averaged statistics of $x(t)$ matching the arbitrarily defined target spatial distribution $p(s)$ as
    well as possible on a finite time horizon and the robotic system sampling measurements in high utility regions specified by
    $p(s)$.

    \begin{figure}[h!]
        \centering
        \includegraphics[width=\linewidth]{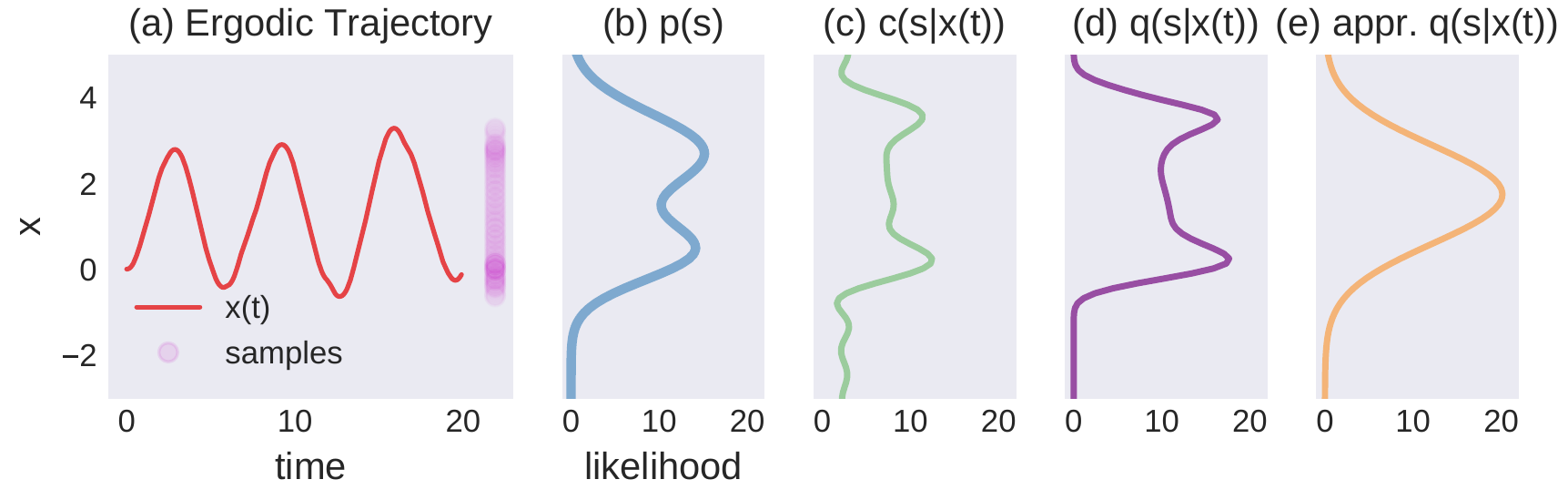}
          \caption{
            (a) Illustration of an ergodic trajectory $x(t)$ with respect to (b) target distribution
            $p(s)$. Time-averaged distribution reconstructions of $x(t)$ are shown using
            Definitions~\ref{def:fourier_recon},~\ref{def:sig_approx}, and
            Eq.~\ref{eq:jensen_sigma}). The Fourier decomposition approach often has residual
            artifacts due to the cosine approximation. $20$ basis functions are used to approximate
            the time-averaged distribution. $\Sigma$-approximation to the time-averaged statistics
            minimizes residual artifacts. (e) Moment matching of the $\Sigma$-approximation provides
            a simplification for computing the time-averaged statistics.
          }
         \label{fig:erg_demo}
    \end{figure}

    As in~\cite{miller2013trajectory, miller2016ergodic}, one can calculate a controller that optimizes the
    ergodic metric such that the trajectory of the robot is ergodic with
    respect to the distribution $p(s)$ (see Fig~\ref{fig:erg_demo} for illustration). However, this
    approach scales $\mathcal{O}(|k|^n)$ where $|k|$ is the maximum integer--valued Fourier term. As
    a result, this method is ill-suited for high-dimensional learning tasks whose exploration states
    are often the full state-space of the robot (often $n > 3$ for most mobile robots). Furthermore,
    the resulting time-averaged distribution reconstruction will often have residual artifacts which
    require additional conditioning to remove. This motivates the following section which defines an
    ergodic measure. \footnote{We make note of the use of the work ``measure'' as
    opposed to metric as we allude to using the KL-divergence which itself is not a metric, but a
    measure.}

\section{KL-Divergence Ergodic Measure} \label{sec:kleee}

    As an alternative to computing the ergodic metric, we present an ergodic measure
    which circumvents the scalability issues mentioned in the previous section. To do this, we
    utilize the Kullback-–Leibler
    divergence~\cite{kullback1951information,kullback1997information,ayvali2017ergodic}
    (KL--divergence) as a measure for ergodicity. Let us first define the approximation to the
    time-averaged statistics of a trajectory $x(t)$:
    \begin{definition}\label{def:sig_approx}
        Given a search domain $\mathcal{S}^v\subset\mathbb{R}^{n+m}$ the $\Sigma$-approximated
        time-averaged statistics of the robot is defined by
        \begin{equation}\label{eq:time-averaged-stats}
            q(s \mid x(t)) = \\ \frac{1}{t_f - t_0} \int_{t_0}^{t_f} \psi(s \mid x(t))dt
        \end{equation}
        where $\psi(s \mid x(t)) = \frac{1}{\eta}\exp \left[
            -\frac{1}{2} \Vert s - \bar{x}(t) \Vert^2_{\Sigma^{-1}}
            \right]$, $\Sigma \in \mathbb{R}^{v \times v}$ is a positive definite matrix parameter that
        specifies the width of the Gaussian, and $\eta$ is a normalization constant.
    \end{definition}
    We call this an approximation because the true time-averaged statistics, as described
    in~\cite{miller2016ergodic} and Definition~\ref{def:fourier_recon}, is a collection of delta
    functions parameterized by time. We approximate the delta function as a Gaussian distribution
    with variance $\Sigma$, converging as $\Vert \Sigma \Vert \to 0$. As an aside, one can treat
    $\Sigma$ as a function of $\bar{x}(t)$ if there is uncertainty in the position of the robot.

    With this approximation, we are able to relax the ergodic objective in~\cite{miller2016ergodic}
    and use the following KL-divergence objective~\cite{ayvali2017ergodic}:
    \begin{align*}
        \footnotesize
        D_\text{KL}(p \Vert q) & = \int_{\mathcal{S}^v} p(s) \log \frac{p(s)}{q(s)} ds \\
        & = \int_{\mathcal{S}^v} p(s)\log p(s) ds -  \int_{\mathcal{S}^v} p(s) \log q(s) ds, \\
        & = -  \int_{\mathcal{S}^v} p(s) \log q(s) ds \\
        & =  - \mathbb{E}_{p(s)} \left[ \log q(s) \right]
    \end{align*}
    where $\mathbb{E}$ is the expectation operator, $q(s) = q(s \mid x(t))$, and $p(s)$ is an
    arbitrary spatial distribution. Note that we drop the first term in the expanded KL-divergence
    because it does not depend on the trajectory of the robot $x(t)$. Rather than computing the
    integral over the exploration space $\mathcal{S}^v$ (partly because of intractability), we
    approximate the expectation operator as
    \begin{align}\label{eq:kl_objective}
        D_\text{KL} (p \Vert q) & =-\mathbb{E}_{p(s)} \left[ \log q(s) \right] \nonumber \\
         & \approx - \sum_{i=1}^N p(s_i) \log q(s_i) \nonumber \\
         & \propto - \sum_{i=1}^N p(s_i) \log\int_{t_0}^{t_f} \exp\left[-\frac{1}{2} \Vert s_i - \bar{x}(t) \Vert^2_{\Sigma^{-1}} \right] dt,
    \end{align}
    where $N$ is the number of samples in the search domain drawn uniformly.\footnote{We can always
    use importance sampling to interchange which distribution we sample from.}  Through this formulation,
    we still obtain the benefits of indirectly sampling from the spatial distribution $p(s)$ without
    having to directly compute derivatives to generate an optimal control signal for the robot.
    Furthermore, this measure prevents computing the measure from scaling drastically with the
    number of exploration states. Figure~\ref{fig:erg_demo} illustrates the resulting reconstruction
    of the time-averaged statistics of a trajectory $x(t)$ using Definition~\ref{def:sig_approx} .
    The following section uses the KL-divergence ergodic measure and derives a controller which
    optimizes~(\ref{eq:kl_objective}) while directly incorporating learned models and policies.

\section{$\text{KL-E}^3$: KL-Ergodic Exploration from Equilibrium} \label{sec:algo}

    In this section, we derive KL-Ergodic Exploration from Equilibrium ($\text{KL-E}^3$), which
    locally optimizes and improves (\ref{eq:kl_objective}). As an additional constraint, we impose
    an equilibrium policy and an approximate transition model of the robot's dynamics on the
    algorithm. By synthesizing $\text{KL-E}^3$ with existing policies that allow the robot to return
    to an equilibrium state (i.e., local linear quadratic regulators (LQR) controller), we can take
    advantage of approximate transition models for planning the robot's motion while
    providing a bound on how unstable the robot can become. We then show how this method is Lyapunov
    attractive~\cite{sontag1999control, khansari2014learning}, allowing the robot to become unstable
    so long as we can ensure the robot will eventually return to an equilibrium.

    \subsection{Model and Policy Assumptions for Equilibrium:}
        We assume that we have a robot whose approximate dynamics can be modeled using the
        continuous time transition model:
        \begin{align}\label{eq:model}
            \dot{x}(t) & = f(x(t),\mu(x(t))) \\ &= g(x(t)) + h(x(t)) \mu(x(t)) \nonumber
        \end{align}
        where $\dot{x}(t) \in \mathbb{R}^n$ is the change of rate of the state $x(t) : \mathbb{R}
        \to \mathbb{R}^n$ of the robot at time $t$, $f(x,u) : \mathbb{R}^{n \times m} \to
        \mathbb{R}^n$ is the (possibly nonlinear) transition model of the robot as a function of
        state $x$ and control $u$ which we partition into $g(x) : \mathbb{R}^n \to \mathbb{R}^n$,
        the free unactuated dynamics, and $h(x) : \mathbb{R}^n \to \mathbb{R}^{n \times m}$, the
        actuated dynamics. In our modeling assumption, we consider a policy $\mu(x) : \mathbb{R}^n
        \to \mathbb{R}^m$ which provides the control signal to the robotic system such that there
        exists a continuous Lyapunov function $V(x)$~\cite{artstein1983stabilization,
        khansari2014learning} which has the following conditions:
        \begin{subequations}\label{eq:lyap_cond}
            \begin{align}
                & V(0) = 0 \\
                \forall x \in \mathcal{B} \backslash \{ 0 \} \quad & V(x) > 0 \\
                \forall x \in \mathcal{B} \backslash \{ 0 \} \quad & \nabla V \cdot f(x, \mu(x)) < 0
            \end{align}
        \end{subequations}
        where $\mathcal{B} \subset \mathbb{R}^n$ is a compact and connected set, and
        \begin{equation}
            \dot{V}(x(t)) = \frac{\partial}{\partial x} V(x) \cdot f(x, u) = \nabla V \cdot f(x, u).
        \end{equation}
        Thus, a trajectory $x(t)$ with initial condition at time $t=t_0$ subject to (\ref{eq:model})
        and $\mu(x)$ is defined as
        \begin{equation}\label{eq:trajectory}
            x(t) = x(t_0) + \int_{t_0}^{t} f(x(t), \mu(x(t))) dt.
        \end{equation}
        For the rest of the paper, we will refer to $\mu(x)$ as an equilibrium policy which is tied
        to an objective which returns the robot to an equilibrium state. In our prior
        work~\cite{abraham2019active} we show how one can include any objective into the synthesis
        of $\text{KL-E}^3$.

    \subsection{Synthesizing a Schedule of Exploratory Actions:}
        Given the assumptions of a known approximate model and an equilibrium policy, our goal is to
        generate a control signal that augments $\mu(x)$ and minimizes (\ref{eq:kl_objective})
        while ensuring $x$ remains within the compact set $\mathcal{B}$ which will allow the
        robot to return to an equilibrium state within the time $t \in \left[t_0, t_f \right]$.

        Our approach starts by quantifying how sensitive (\ref{eq:kl_objective}) is to switching
        from the policy $\mu(x(t))$ to an arbitrary control vector $\mu_\star(t)$ at any time $\tau
        \in \left[ t_0 , t_f \right]$ for an infinitesimally small duration of time $\lambda$. We
        will later use this sensitivity to calculate a closed-form solution to the most influential
        control signal $\mu_\star(t)$.
        \begin{proposition}
            The sensitivity of (\ref{eq:kl_objective}) with respect to the duration time $\lambda$,
            of switching from the policy $\mu(x)$ to an arbitrary control signal $\mu_\star(t)$ at
            time $\tau$ is
            \begin{equation}\label{eq:mode_insertion}
                \frac{\partial D_\text{KL}}{\partial \lambda} = \rho(\tau)^\top (f_2 - f_1)
            \end{equation}
            where $f_2 = f(x(\tau), \mu_\star(\tau))$ and $f_1 = f(x(\tau), \mu(x(\tau))$, $\rho(t)
            \in \mathbb{R}^n$ is the adjoint, or co-state variable which is the solution of the
            following differential equation
            \begin{equation}\label{eq:adjoint_diff}
                \dot{\rho}(t) =
                    \sum_i \frac{p(s_i)}{q(s_i)} \frac{\partial \psi}{\partial x}
                     -
                     \left(
                    \frac{\partial f}{\partial x} + \frac{\partial f}{\partial u} \frac{\partial \mu}{\partial x}
                 \right)^\top \rho(t)
            \end{equation}
            subject to the terminal constraint $\rho(t_f) = \mathbf{0}$, and $\frac{\partial \psi}{\partial x}$ is
            evaluated at each sample $s_i$.
        \end{proposition}
        \begin{proof}
            See Appendix~\ref{app:proofs}
        \end{proof}

        The sensitivity $\frac{\partial}{\partial \lambda} D_\text{KL}$ is known as the mode
        insertion gradient~\cite{axelsson_JOTA_modeinsertion}. Note that the second term in
        (\ref{eq:adjoint_diff}) encodes how the dynamics will change subject to the policy $\mu(x)$.
        We can directly compute the mode insertion gradient for any control $\mu_\star(t)$ that we
        choose. However, our goal is to find a schedule of $\mu_\star(t)$ which minimizes
        (\ref{eq:kl_objective}) while still bounded by the equilibrium policy $\mu(x)$. We solve for
        this augmented control signal by formulating the following optimization problem:
        \begin{equation}\label{eq:aux_obj}
            \mu_\star(t) = \argmin_{\mu(t) \forall t \in [t_0, t_f]} \int_{t_0}^{t_f} \frac{\partial}{\partial \lambda} D_\text{KL} \Big |_{\tau=t} +\frac{1}{2} \Vert \mu_\star(t) - \mu(x(t)) \Vert_\mathbf{R}^2 dt
        \end{equation}
        where $\mathbf{R}\in \mathbb{R}^{m \times m}$ is a positive definite matrix that penalizes
        the deviation from the policy $\mu(x)$ and $\frac{\partial}{\partial \lambda}
        D_\text{KL}|_{\tau=t}$ is (\ref{eq:mode_insertion}) evaluated at time $t$.
        \begin{proposition}
            The augmented control signal $\mu_\star(t)$ that minimizes (\ref{eq:aux_obj}) is given by
            \begin{equation} \label{eq:explr_actions}
                \mu_\star(t) = - \mathbf{R}^{-1} h(x(t))^\top \rho(t) + \mu(x(t)).
            \end{equation}
        \end{proposition}
        \begin{proof}
            Taking the derivative of (\ref{eq:aux_obj}) with respect to $\mu_\star(t)$ at each
            instance in time $t \in [t_0 , t_f]$ gives
            \begin{align}\label{eq:j2dmu}
                \int_{t_0}^{t_f}\frac{\partial}{\partial \mu_\star} & \left( \frac{\partial}{\partial \lambda} D_\text{KL} \Big |_{\tau=t} +\frac{1}{2} \Vert \mu_\star(t) - \mu(x(t)) \Vert_\mathbf{R}^2 \right) dt \\
                & =  \int_{t_0}^{t_f}  h(x(t))^\top \rho(t) + \mathbf{R} (\mu_\star (t) - \mu(x(t))) dt \nonumber
            \end{align}
            where we expand $f(x,u)$ using (\ref{eq:model}). Since the expression under the integral
            in (\ref{eq:aux_obj}) is convex in $\mu_\star(t)$ and is at an optimizer when
            (\ref{eq:j2dmu}) is equal to $\mathbf{0} \forall t\in [t_0, t_f]$, we set the expression
            in (\ref{eq:j2dmu}) to zero and solve for $\mu_\star(t)$ at each instant in time giving
            us
            \begin{equation*}
                \mu_\star(t) = - \mathbf{R}^{-1}h(x(t))^\top \rho(t) + \mu(x(t))
            \end{equation*}
            which is the schedule of exploratory actions that reduces the objective for time $t \in \left[t_0 , t_f \right]$
            and is bounded by $\mu(x)$.
        \end{proof}
        In practice, the first term in (\ref{eq:explr_actions}) is calculated and applied to the
        robot using a true measurement of the state $\hat{x}(t)$ for the policy $\mu(x)$. We refer
        to this first term as $\delta \mu_\star(t) = - \mathbf{R}^{-1}h(x(t))^\top \rho(t)$ yielding
        $\mu_\star(t) = \delta \mu_\star(t) + \mu(\hat{x}(t))$.

        Given the derivation of the augmented control signal that can generate ergodic exploratory
        motions, we verify the following through theoretical analysis in the next section:
        \begin{itemize}
            \item that (\ref{eq:explr_actions}) does in fact reduce (\ref{eq:kl_objective})
            \item that (\ref{eq:explr_actions}) imposes a bound on the conditions in (\ref{eq:lyap_cond})
            \item and that a robotic system subject to (\ref{eq:explr_actions}) has a notion of Lyapunov attractiveness
        \end{itemize}

    \textbf{Theoretical Analysis:}
        We first illustrate that our approach for computing (\ref{eq:explr_actions}) does reduce
        (\ref{eq:kl_objective}).
        \begin{corollary}
            Let us assume that $\frac{\partial}{\partial \mu}\mathcal{H}\neq 0$ $\forall t \in \left[ t_0, t_f \right]$,
            where $\mathcal{H}$ is the control Hamiltonian.
            Then
            \begin{equation}
                \frac{\partial}{\partial \lambda} D_\text{KL} = -\Vert h(x(t))^\top \rho(t) \Vert_{\mathbf{R}^{-1}}^2< 0
            \end{equation}
            $\forall t \in \left[t_0, t_f\right]$ subject to $\mu_\star(t)$.
        \end{corollary}
        \begin{proof}
            Inserting (\ref{eq:explr_actions}) into (\ref{eq:mode_insertion}) and dropping the dependency of time for clarity gives
            \begin{align}\label{eq:neg-djdlam}
                \frac{\partial}{\partial \lambda} D_\text{KL} &= \rho(t) ^\top \left( f_2 - f_1\right) \nonumber \\
                & = \rho^\top\left(
                g(x) + h(x) \mu_\star - g(x) - h(x) \mu(x)
                \right) \nonumber \\
                & = \rho^\top( -h(x)\mathbf{R}^{-1} h(x)^\top \rho + h(x)\mu(x) - h(x)\mu(x)) \nonumber \\
                & = - \rho^\top h(x) \mathbf{R}^{-1} h(x)^\top \rho \nonumber \\
                &  = - \Vert h(x(t))^\top \rho(t) \Vert_{\mathbf{R}^{-1}}^2 \le 0.
            \end{align}
            Thus, $\frac{\partial}{\partial \lambda} D_\text{KL}$ is always negative subject to (\ref{eq:explr_actions}).
        \end{proof}
        For $\lambda>0$ we can approximate the reduction in $D_\text{KL}$ as $\Delta D_\text{KL}
        \approx \frac{\partial}{\partial \lambda} D_\text{KL} \lambda \le 0$. Thus, by
        applying (\ref{eq:explr_actions}), we are generating exploratory motions that minimize the
        ergodic measure defined by (\ref{eq:kl_objective}).

        Our next set of analysis involves searching for a bound on the conditions in
        (\ref{eq:lyap_cond}) when (\ref{eq:explr_actions}) is applied at any time $\tau \in \left[0,
        t-\lambda \right]$ for a duration $\lambda \le t$.
        \begin{theorem}\label{tmh1}
            Given the conditions in (\ref{eq:lyap_cond}) for a policy $\mu(x)$,
            then $V(x^\tau_\lambda(t)) - V(x(t)) \le \lambda \beta$,
            where $x^\tau_\lambda(t)$ is the solution to (\ref{eq:switched_traj}) subject to (\ref{eq:explr_actions})
            for $\tau \in \left[ 0, t-\lambda \right]$, $\lambda \le t$, and
            \begin{equation}
                \beta = \sup_{t \in \left[ \tau, \tau + \lambda \right]}
            - \nabla V \cdot h(x(t)) \mathbf{R}^{-1}h(x(t))^\top \rho(t).
            \end{equation}
        \end{theorem}
        \begin{proof}
            See Appendix~\ref{app:proofs}.
        \end{proof}

        We can choose any time $\tau \in \left[0, t-\lambda\right]$ to apply $\mu_\star(t)$ and
        provide an upper bound quantifying the change of the Lyaponov function described in
        (\ref{eq:lyap_cond}) by fixing the maximum value of $\lambda$ during active exploration. In
        addition, we can tune $\mu_\star(t)$ using the regularization value $\mathbf{R}$ such that
        as $\Vert \mathbf{R} \Vert \to \infty$, $\beta \to 0$ and $\mu_\star(t) \to \mu(x(t))$.

        Given this bound, we can guarantee Lyapunov attractiveness~\cite{polyakov2014stability}, where there exists a
        time $t$ such that the system (\ref{eq:trajectory}) is guaranteed to return to a region of attraction (from
        which the system can be guided towards a stable equilibrium state $x_0$).
        \begin{definition}
            A robotic system defined by (\ref{eq:model}) is Lyapunov attractive if at some time $t$,
            the trajectory of the system $x(t) \in \mathcal{C}(t) \subset \mathcal{B}$ where
            $\mathcal{C}(t) = \{ x(t) \vert V(x) \le \beta^\star, \nabla V \cdot f(x(t), \mu(x(t)))
            <0\}$, $\beta^\star > 0$ is the maximum level set of $V(x)$ where $\nabla V \cdot f(x,
            \mu(x)) < 0$, and $\lim_{t\to\infty}  x(t)  \to x_0$ such that $x_0$ is an equilibrium
            state.
        \end{definition}
        \begin{theorem}\label{thm2}
            Given the schedule of exploratory actions (\ref{eq:explr_actions})  $\forall t \in
            \left[\tau, \tau + \lambda \right]$, a robotic system governed by (\ref{eq:model}) is
            Lyapunov attractive such that $\lim_{t\to\infty} x^\tau_\lambda(t) \to x_0$.
        \end{theorem}
        \begin{proof}
            Using Theorem~\ref{tmh1}, the integral form of the Lyapunov function
            (\ref{eq:lyap_switch}), and the identity (\ref{eq:lyap_chain}), we can write
            \begin{align*}
                V(x^\tau_\lambda(t)) & = V(x(0))
                        + \int_{0}^{t} \nabla V \cdot f(x(s), \mu(x(s)) ) ds \nonumber \\
                        & \quad \quad - \int_{\tau}^{\tau + \lambda} \nabla V \cdot h(x(s)) \mathbf{R}^{-1} h(x(s))^\top \rho(s) ds
            \end{align*}
            \begin{equation*}
                \le V(x(0)) - \gamma t + \beta \lambda < \beta^\star,
            \end{equation*}
            where
            \begin{equation}
                -\gamma = \sup_{s \in \left[0, t \right]} \nabla V \cdot f(x(s), \mu(x(s))) < 0.
            \end{equation}
            Since $\lambda$ is fixed and $\beta$ can be tuned by the matrix weight $\mathbf{R}$, we can choose a $t$ such that
            $\gamma t \gg \beta \lambda$.
            Thus, $\lim_{t\to \infty} V(x^\tau_\lambda(t)) \to V(x_0)$ and $\lim_{t\to\infty}  x^\tau_\lambda(t)  \to x_0$, implies
            Lyapunov attractiveness,  where $V(x_0)$ is the minimum of the Lyapunov function at the equilibrium state $x_0$.

        \end{proof}
        Proving Lyapunov attractiveness allows us to make the claim that a robot will return to a
        region where $\dot{V}(x)<0$ subject to the policy $\mu(x)$. This enables the robot to
        actively explore states which would not naturally have a safe recovery. Moreover, this
        analysis shows that we can choose the value of $\lambda$ and $\mathbf{R}$ when calculating
        $\mu_\star(t)$ such that attractiveness always holds, giving us an algorithm that is safe for
        active learning.

        So far, we have shown that (\ref{eq:explr_actions}) is a method that generates approximate
        ergodic exploration from equilibrium policies. We prove that this approach does reduce
        (\ref{eq:kl_objective}) and show the ability to quantify and bound how much the active
        exploration process will deviate the robotic system from equilibrium. Last, it is shown that
        generating data for learning does not require constant guaranteed Lyapunov stability of the
        robotic system, but instead introduce the notion of attractiveness where we allow the robot
        to explore the physical realm so long as the time to explore is finite and the magnitude of
        the exploratory actions is restrained. In the following section, we extend our previous work
        in~\cite{abraham2019active} by further approximating the time averaged-statistics so that
        computing the adjoint variable can be done more efficiently.

    \subsection{$\text{KL-E}^3$ for Efficient Planning and Exploration}

        We extend our initial work by providing a further approximation to computing the
        time-averaged statistics which improves the computation time of our implementation. Taking
        note of (\ref{eq:kl_objective}), we can see that we have to evaluate $q(s_i)$ at each sample
        point $s_i$ where $q(s) = q(s \vert x(t))$ has to evaluate the stored trajectory at each
        time. In real robot experiments, often the underlying spatial statistics $p(s)$ change
        significantly over time. In addition, most robot experiments require replanning which
        results in lost information over repeated iterations. Thus, rather than trying to compute
        the whole time averaged trajectory in Definition~\ref{def:sig_approx}, we opt to approximate
        the distribution by applying Jensen's inequality:
        \begin{align}\label{eq:jensen_sigma}
            q(s \vert x(t)) & \propto \int_{t_0}^{t_f} \exp\left[ -\frac{1}{2} \Vert s - \bar{x}(t)
            \Vert_{\Sigma^{-1}}^2\right]dt \nonumber \\
            & \ge \exp\left( -\frac{1}{2} \int_{t_0}^{t_f} \Vert s -
            \bar{x}(t) \Vert_{\Sigma^{-1}}^2 dt \right).
        \end{align}
        Using this expression in (\ref{eq:kl_objective}), we can write
        \begin{align}
            D_\text{KL} &\propto -\int_{\mathcal{S}^v} p(s)\log q(s) ds \nonumber \\
                        &= -\int_{\mathcal{S}^v} p(s)\log \left(\exp\left( -\frac{1}{2} \int_{t_0}^{t_f} \Vert s - \bar{x}(t) \Vert_{\Sigma^{-1}}^2 dt \right) \right)ds \nonumber \\
                        &\propto \int_{\mathcal{S}^v} p(s) \left(\int_{t_0}^{t_f}\Vert s - \bar{x}(t) \Vert_{\Sigma^{-1}}^2 dt \right) ds \nonumber \\
                        &\approx \sum_{i}^N p(s_i)\int_{t_0}^{t_f}\Vert s_i - \bar{x}(t) \Vert_{\Sigma^{-1}}^2 dt.
        \end{align}
        Following the results to compute (\ref{eq:mode_insertion}), we can show
        that (\ref{eq:explr_actions}) remains the same where the only modification is in the adjoint
        differential equation where
        \begin{equation}
            \dot{\rho}(t) = -\sum_i p(s_i) \frac{\partial \ell}{ \partial x} - \left(\frac{\partial f}{\partial x} + \frac{\partial f}{\partial u} \frac{\partial \mu}{\partial x}\right)^\top \rho(t)
        \end{equation}
        such that $\ell = \ell(s, x) = \Vert s - x \Vert_{\Sigma^{-1}}^2$. This formulation has no need to
        compute $q(s)$ and instead only evaluate $p(s)$ at sampled points. We reserve using this
        implementation for robotic systems that are of high dimensional space or when calculating the
        derivatives can be costly due to over-parameterization (i.e., multi-layer networks). Note
        that all the theoretical analysis still holds because the fundamental theory relies on the
        construction through hybrid systems theory rather than the KL-divergence itself. The
        downside to this approach is that one now loses the ability to generate consistent ergodic
        exploratory movements. The effect can be seen in Figure~\ref{fig:erg_demo}(e) where the
        trajectory is approximated by a wide Gaussian\textemdash rather than the bi-model
        distribution found in Fig.~\ref{fig:erg_demo}(d). However, for non-stationary $p(s)$, having
        exact ergodic behavior is not necessary and such approximations at the time-averaged
        distribution level are sufficient.

        In the following subsection, we provide base algorithm and implementation details for
        $\text{KL-E}^3$ and present variations based on the learning goals in~\ref{sec:ex} in the
        Appendix.

    \subsection{Algorithm Implementation}

        In this section, we provide an outline of a base implementation of $\text{KL-E}^3$ in
        Algorithm~\ref{alg:KLE3}. We also define some variables which were not previously mentioned in
        the derivation of $\text{KL-E}^3$ and provide further implementation detail.

        There exist many ways one could use (\ref{eq:explr_actions}). For instance, it is possible to
        simulate a dynamical system for some time horizon $t_H$ and apply (\ref{eq:explr_actions}) in a
        trajectory optimization setting. Another way is to repeatedly generate trajectory plans at each
        instance and apply the first action in a model-based predictive control (MPC) manner. We found
        that the choice of $\tau$ and $\lambda$ can determine how one will apply
        (\ref{eq:explr_actions}). That is, given some time $t_i$, if $\tau=t_i$ and $\lambda=t_H$, we
        recover the trajectory optimization formulation whereas when $\tau=t_i$ and $\lambda=dt$, where $dt$ is
        the time step, then the MPC formulation is recovered.\footnote{One can also automate choosing
        $\tau$ and $\lambda$ using a line search~\cite{more1994line}.} Rather than focusing on incremental
        variations, we focus on the general structure of the underlying algorithm when combined with a
        learning task.

        We first assume that we have an approximate transition model $f(x,u)$ and an equilibrium policy
        $\mu(x)$. A simulation time horizon $t_H$ and a time step $dt$ is specified where the true robot
        measurements of state are given by $\hat{x}(t)$ and the simulated states are $x(t)$. Last, a
        spatial distribution $p(s)$ is initialized (usually uniform to start), and an empty data set
        $\mathcal{D} = \{\hat{x}(t_j), y(t_j) \}_j$ is initialized where $y(t)$ are measurements.

        Constructing $p(s)$ will vary depending on the learning task itself. The only criteria that is
        necessary for $p(s)$ is that it depends on the measurement data that is collected and used in the learning task.
        Furthermore, $p(s)$ should represent a utility function that indicates where informative measurements are in the
        search space to improve the learning task. In the following sections, we provide examples for constructing and
        updating $p(s)$ given various learning tasks.

        Provided the initial items, the algorithm first samples the robot's state $\hat{x}(t_i)$ at the
        current time $t_i$. Using the transition model and policy, the next states $x(t)
        \forall t \in [t_i, t_i+t_H]$ are simulated. A set of $N$, samples $(s_1, s_2, \ldots, s_N)$ are
        generated and used to compute $p(s), q(s)$. The adjoint variable is then backwards simulated
        from $t = t_i+t_H \to t_i$ and is used to compute $\delta\mu_\star(t)$. We ensure robot safety
        by applying $\delta \mu_\star(t) + \mu(\hat{x}(t))$ with real measurements of the robot's state.
        Data is then collected and appended to $\mathcal{D}$ and used to update $p(s)$. Any additional
        steps are specified by the learning task. The pseudo-code for this description is provided
        in Algorithm~\ref{alg:KLE3} .

        \begin{algorithm}[!h]
            \caption{$\text{KL-E}^3$ Base Algorithm}
            \begin{algorithmic}[1]
                \State \textbf{init:} approximate transition model $f(x,u)$,
                        initial true state $\hat{x}(0)$, equilibrium policy $\mu(x)$,
                        spatial distribution $p(s)$, simulation time horizon $t_H$, time step $dt$.
                        data set $\mathcal{D}$, $i=0$
                \While{task not done}
                    \State set $x(t_i) = \hat{x}(t_i)$
                    \State $\triangleright$ simulation loop
                    \For{$\tau_i \in \left[t_i, \ldots, t_i + t_H \right]$}
                        \State \qquad $\triangleright$ forward predict states using any
                        \State \qquad $\triangleright$ integration method (Euler shown)
                        \State $x(\tau_{i+1}) = x(\tau_i) + f(x(\tau_i), \mu(x(\tau_i))) dt$
                    \EndFor
                    \State \qquad $\triangleright$ backwards integrate choosing $\dot{\rho}(t)$
                    \State \qquad $\triangleright$ set the terminal condition
                    \State generate $N$ samples of $s_i$ uniformly within $\mathcal{S}^v$
                    \State $\rho(t_i + t_H) = \mathbf{0}$
                    \For{$\tau_i \in \left[t_H + t_i, \ldots, t_i \right]$}
                        \State $\rho(\tau_{i-1}) = \rho(\tau_i) - \dot{\rho}(\tau_i) dt$
                        \State \qquad $\triangleright$ since $x(t)$ is simulated, we return
                        \State \qquad $\triangleright$ just the first term of ($\ref{eq:explr_actions}$)
                        \State \qquad $\triangleright$ and calculate $\mu(x)$ online
                        \State $\delta \mu_\star(\tau_{i-1}) = -\mathbf{R}^{-1} h(x(\tau_{i-1}))^\top \rho(\tau_{i-1})$
                    \EndFor
                    \State \qquad $\triangleright$ apply to real robot
                    \State chose $\tau \in [t_i, t_i+t_H]$ and $\lambda \le t_H$ or use line search~\cite{more1994line}
                    \For{$t \in [t_i, t_{i+1}] $}
                        \If{ $t \in [\tau, \tau + \lambda]$}
                            \State apply $\mu_\star(t) = \delta\mu_\star(t) + \mu(\hat{x}(t))$
                        \Else
                            \State apply $\mu(x(t))$
                        \EndIf
                        \If{time to sample}
                            \State measure true state $\hat{x}(t)$ and measurements $y(t)$
                            \State append to data set $\mathcal{D} \gets \{ \hat{x}(t), y(t)\}$
                        \EndIf
                    \EndFor
                    \State update $p(s)$ given $\mathcal{D}$ \Comment task specific
                    \State update $f(x,u), \mu(x)$ \Comment if needed
                    \State update learning task
                    \State $i \gets i + 1$
                \EndWhile
            \end{algorithmic}
            \label{alg:KLE3}
        \end{algorithm}

\section{Example Learning Goals} \label{sec:ex}

        \begin{figure*}[th!]
            \centering
            \begin{subfigure}{0.22\textwidth}
                \includegraphics[width=\linewidth]{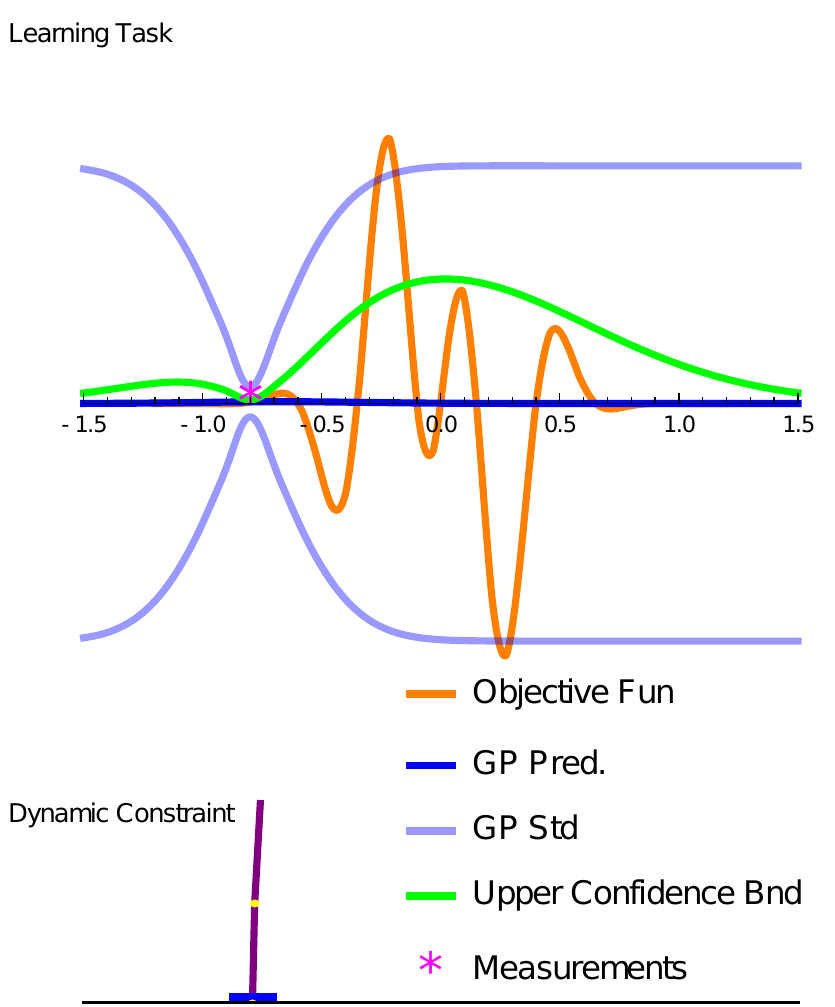}
                \caption{$t = 0$}
            \end{subfigure}
            \begin{subfigure}{0.24\textwidth}
                \includegraphics[width=\linewidth]{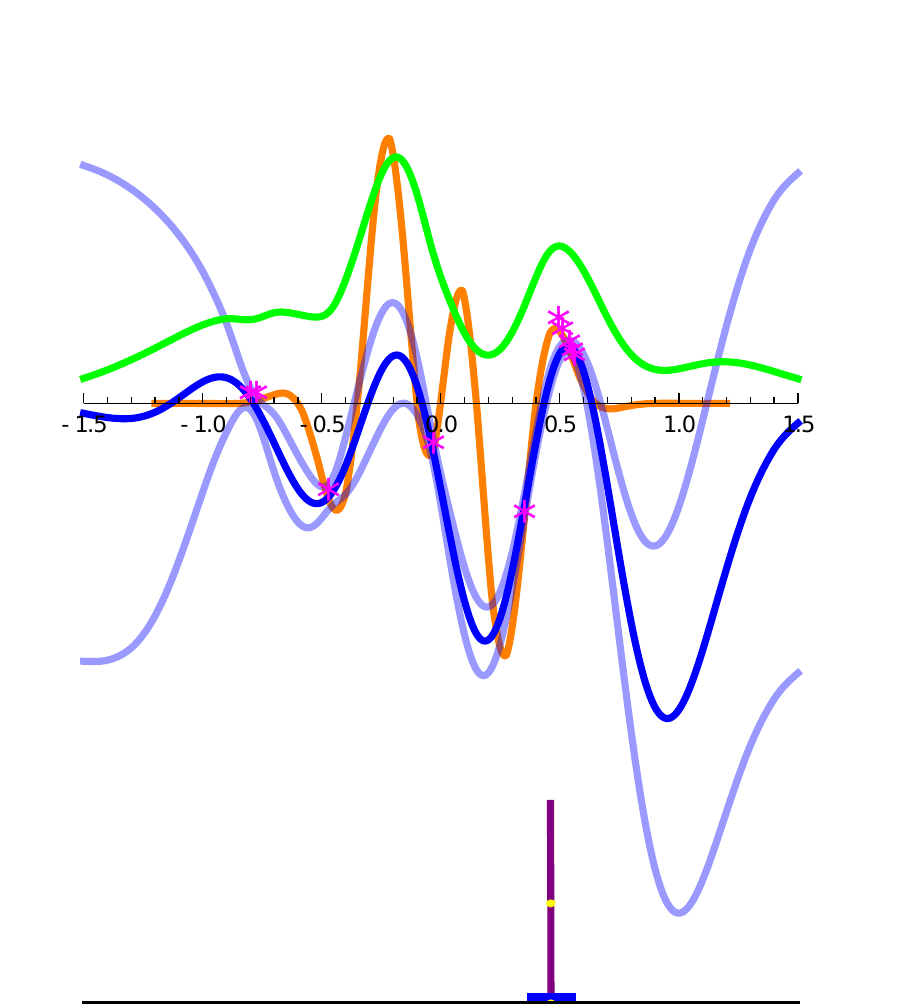}
                \caption{$t = 10$}
            \end{subfigure}
            \begin{subfigure}{0.24\textwidth}
                \includegraphics[width=\linewidth]{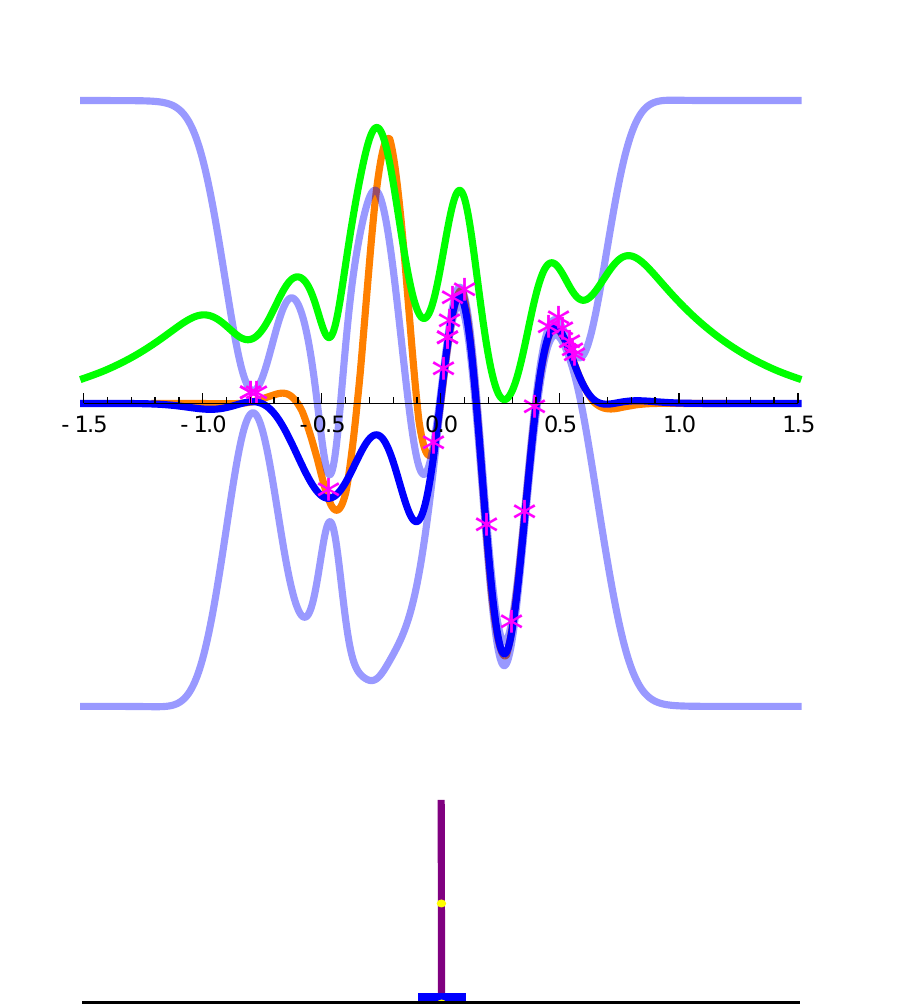}
                \caption{$t = 20$}
            \end{subfigure}
            \begin{subfigure}{0.24\textwidth}
                \includegraphics[width=\linewidth]{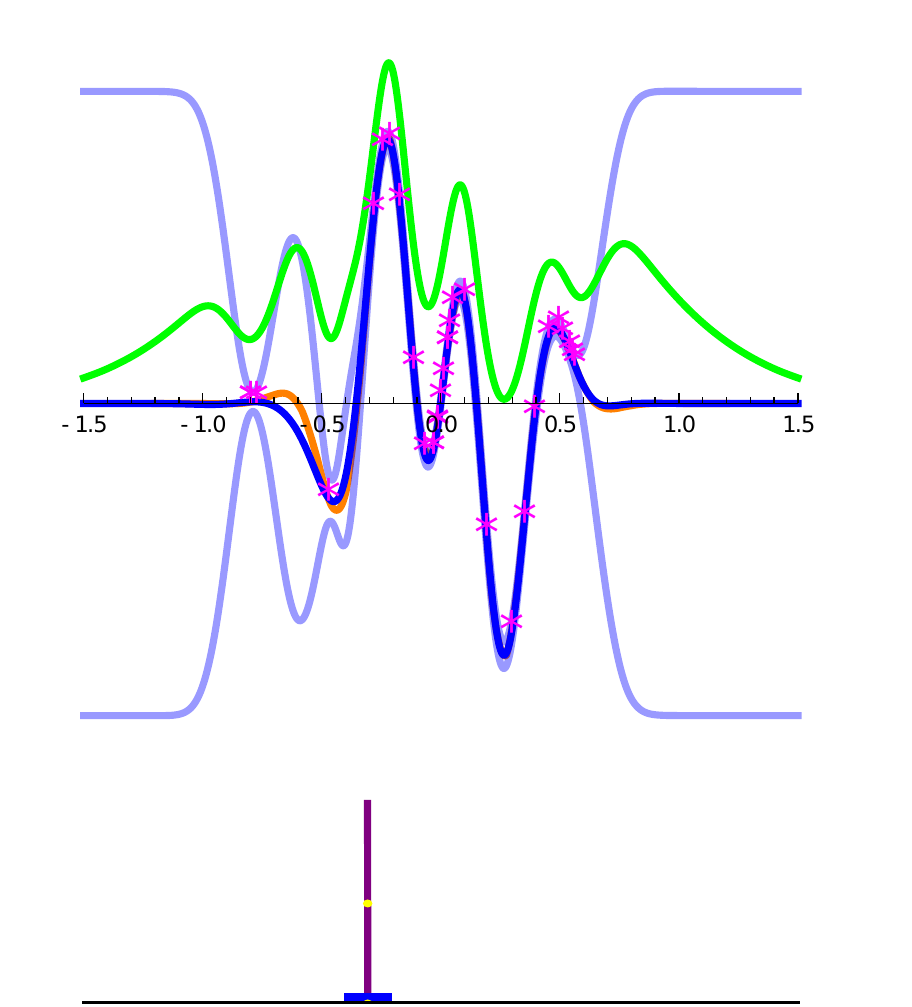}
                \caption{$t = 30$}
            \end{subfigure}
            \caption{
                    Time series snap-shots of cart double pendulum actively sampling and estimating
                    the objective function (orange). The uncertainty (light blue) calculated from the
                    collected data set drives the exploratory motion of the cart double pendulum
                    while our method ensures that the cart double pendulum is maintained in its
                    upright equilibrium state.
            }
            \label{fig:bayes_opt_time_series}
        \end{figure*}

    In this section, our goal is to use $\text{KL-E}^3$ for improving example methods for
    learning. In particular, we seek to show that our method can improve Bayesian optimization,
    transition model learning (also known as dynamics model learning or system identification), and
    off-policy robot skill learning. In each subsection, we provide an overview of the learning
    goal and define the spatial distribution $p(s)$ used in our method. In addition,
    we show the following:
    \begin{itemize}
        \item that our method is capable of improving the learning process through exploration
        \item that our method does not violate equilibrium policies and destabilize the robot
        \item and that our method efficiently explores through exploiting the dynamics of a robot
            and the underlying spatial distribution.
    \end{itemize}
    For each example, we provide implementation, including parameters used, in the appendix.

    \subsection{Bayesian Optimization}\label{sec:bayes_opt}

        In our first example, we explore $\text{KL-E}^3$ for Bayesian optimization using a cart
        double pendulum system~\cite{zhong2001energy} that needs to maintain itself at the upright
        equilibrium. Bayesian optimization is a probabilistic approach for optimizing
        objective functions $\phi(x) : \mathbb{R}^n \to \mathbb{R}$ that are either expensive to
        evaluate or are highly nonlinear. A probabilistic model (often a Gaussian process) of the
        objective function is built from sampled data $x_k\in \mathbb{R}^n$ and the posterior of the
        model is used to construct an acquisition function~\cite{snoek2012practical}. The
        acquisition function maps the sample space $x$ to a value which indicates the utility of the
        sample (in other words, how likely is the sample to provide information about the objective
        given the previous sample in $x$). The acquisition function is often simpler and easier to
        calculate for selecting where to sample next rather than the objective function itself.
        Assuming one can freely sample the space $x$, Bayesian optimization takes a sample based on
        the acquisition function and a posterior is computed. The process then repeats until some
        terminal number of evaluations of the objective or the optimizer is reached. We provide
        pseudo-code for Bayesian optimization in the Appendix, Alg~\ref{alg:bayes_opt}.

        In many examples of Bayesian optimization, the assumption is that the learning algorithm can
        freely sample anywhere in the sample space $x\in\mathbb{R}^n$; however, this is not always
        true. Consider an example where a robot must collect a sample from a Bayesian optimization
        step where the search space of this sample intersects the state-space of the robot itself.
        The robot is constrained by its dynamics in terms of how it can sample the objective.
        Thus, the Bayesian optimization step becomes a constrained optimization problem where the goal
        is to reach the optimal value of the acquisition function subject to the dynamic constraints
        of the robot. Furthermore, assume that the motion of the robot is restricted to maintain the
        robot at an equilibrium (such as maintaining the inverted equilibrium of the cart double
        pendulum). The problem statement is then to enable a robot to execute a sample step of
        Bayesian optimization by taking into account the constraints of the robot. We use this
        example to emphasize the effectiveness of our method for exploiting the local dynamic
        information using a cart double pendulum where the equilibrium state is at the
        upright inverted state and a policy maintains the double pendulum upright.

        Here, we use a Gaussian process with the radial basis function (RBF) to build a model of the
        objective function shown in Fig.~\ref{fig:bayes_opt_time_series}. Using Gaussian process predictive
        posterior mean $\bar{\mu}(x)$ and variance $\sigma(x)$, the upper confidence bound
        (UCB)~\cite{snoek2012practical} acquisition
        function is defined as
        \begin{equation}
            \text{UCB}(x) = \bar{\mu}(x) + \kappa \sigma(x)
        \end{equation}
        where $\kappa >0$. We augment Alg~\ref{alg:KLE3} for Bayesian optimization by setting the
        UCB acquisition function as the target distribution which we define through the Boltzmann
        softmax function, a common method of
        converting functions that indicate regions of high-value into distributions~\cite{bishop2006pattern, sutton2018reinforcement}:
        \footnote{Other distribution forms are possible, but the analysis of their effects is left for future work
        and we choose the Boltzmann softmax formulation for consistency throughout each example.}
        \begin{equation}\label{eq:ucb_target}
            p(s) = \frac{\exp( c \text{UCB}(s) )}{\int_{\mathcal{S}^v}\exp( c \text{UCB}(\bar{s}) ) d \bar{s}}
        \end{equation}
        where $c>0$ is a scaling constant. Note that the denominator is approximated as a sum over
        the samples that our method generates. An approximate linear model of the cart double
        pendulum dynamics centered around the unstable equilibrium is used along with an LQR policy
        that maintains the double pendulum upright. We provide brief pseudo-code of the base
        Algorithm~\ref{alg:KLE3} in the Appendix Alg.~\ref{alg:KLE3_bayes_opt}.

        \begin{figure}
            \centering
            \includegraphics[width=0.9\linewidth]{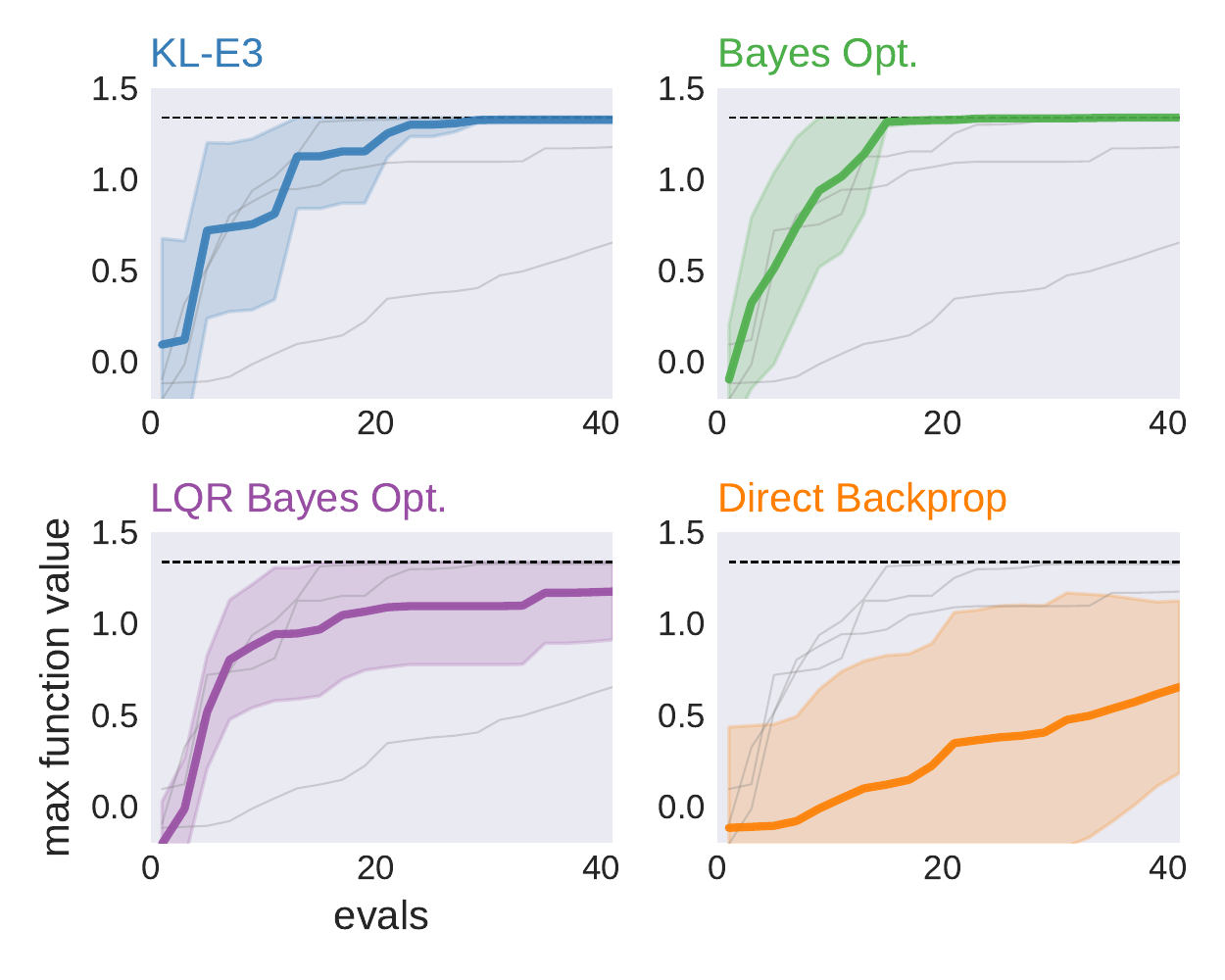}
            \caption{
                Comparison of $\text{KL-E}^3$ against Bayesian Optimization without dynamic
                constraint, LQR-Bayesian optimization, and direct maximization of the acquisition
                function through gradient propagation of the cart double pendulum approximate
                dynamics in determining the maximum value of the objective function through
                exploration. Our method is able to perform as well as Bayesian optimization directly
                sampling the exploration space while performing better than the naive LQR-Bayesian
                optimization. Dashed black line indicates the maximum value of the function.
            }
            \label{fig:bayes_opt_comparison}
        \end{figure}

        \begin{figure}[h!]
            \centering
            \includegraphics[width=0.8\linewidth]{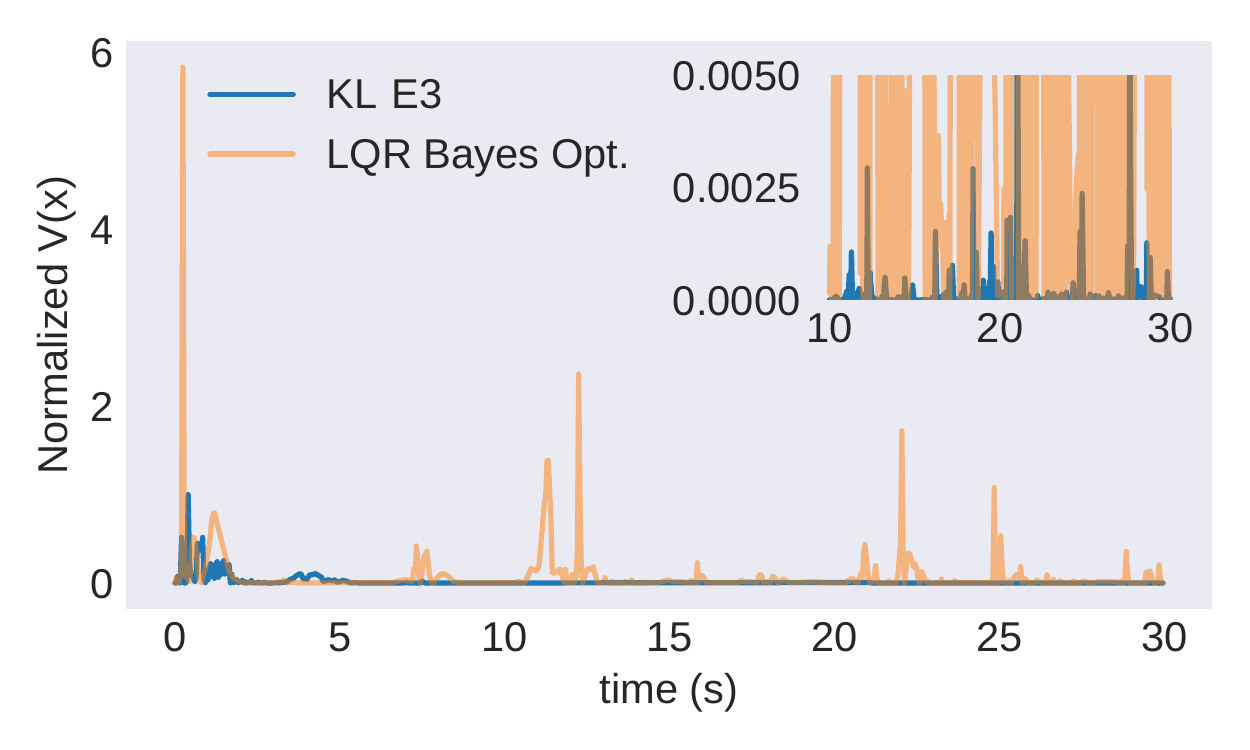}
            \caption{
                Normalized Lyapunov function for the cart double pendulum with upright equilibrium.
                Our method (shown in blue) is roughly 6 times more stable than LQR-Bayesian optimization. The large
                initial values indicate the sweeping shown in
                Fig.~\ref{fig:bayes_opt_time_series}(a) when the cart double pendulum moves across
                the search space. Subsequent application of the exploratory motion refine the
                exploration process. The Lyapunov attractiveness property is enforced through
                automatic switching of the exploration process.
            }
            \label{fig:lyap}
        \end{figure}

        We first illustrate that our method generates ergodic exploration through an execution of
        our method for Bayesian optimization in Figure~\ref{fig:bayes_opt_time_series}. Here, the
        time-series evolution of $\text{KL-E}^3$ is shown to sample proportional to the acquisition
        function. As a result, our method generates samples near each of the peaks of the objective
        function. Furthermore, we can see that our method is exploiting the dynamics as well as the
        equilibrium policy, maintaining Lyapunov attractiveness with respect to the inverted
        equilibrium (we will later discuss numerical results in Fig~\ref{fig:lyap}).

        Next, our method is compared against three variants of Bayesian optimization: the first is
        Bayesian optimization with no dynamics constraint (i.e., no robot is used); second, a linear
        quadratic regulator (LQR) variation of Bayesian optimization where the maximum of the
        acquisition function is used as a target for an LQR controller; and last a direct
        maximization of the acquisition using the stabilizing equilibrium policy
        (see~\cite{abraham2019activelearning} for detail) is used. A total of 10 trials for each
        method are collected with the agent starting at the same location uniformly sampled between
        $-0.8$ and $0.8$ throughout the sample space. In Fig.~\ref{fig:bayes_opt_comparison} we show
        that our method not only performs comparably to Bayesian optimization without dynamic
        constraints\footnote{This may change if the dynamics of the robot are slower or the
        exploration space is sufficiently larger. Note that the other methods would also be equally
        affected.}, but outperforms both LQR and direct maximization variants of Bayesian
        optimization. Because LQR-Bayes method does not take into account dynamic coverage, and
        instead focuses on reaching the next sample point, the dynamics of the robot often do not
        have sufficient time to stabilize which leads to higher variance of the learning objective.
        We can see this in Fig.~\ref{fig:lyap} where we plot a Lyapunov function for the cart double
        pendulum~\cite{zhong2001energy} at the upright unstable equilibrium. Specifically, our
        method is roughly 6 times more stable at the start of the exploration compared to the LQR
        variant of Bayesian optimization. Lyapunov attractiveness is further illustrated in
        Fig.~\ref{fig:lyap} as time progresses and each exploratory motion is closer to the
        equilibrium. Last, directly optimizing the highly nonlinear acquisition function often leads
        to local optima, yielding poor performance in the learning goal. This can be seen with the
        performance of directly optimizing UCB using the cart double pendulum approximate
        dynamics in Fig.~\ref{fig:bayes_opt_comparison}) where the cart double pendulum would often
        find and settle at a local optima.

        In this example, the cart double pendulum only needed to explore the cart position domain to
        find the maximum of the objective function. The following example illustrates a more dynamic
        learning goal where the robot needs to generate a stochastic model of its own dynamics through
        exploration within the state-space.

    \subsection{Stochastic Transition Model Learning}

        In this next example $\text{KL-E}^3$ is used to collect data for learning a stochastic
        transition model of a quadcopter~\cite{fan2017online} dynamical system by exploring the
        state-space of the quadcopter while remaining at a stable hover. Our goal is to show that
        our method can efficiently and effectively explore the state-space of the quadcopter
        (including body linear and angular velocities) in order to generate data for learning a
        transition model of the quadcopter for model-based control. In addition, we show that
        the exploratory motions improve the quality of data generated for
        learning while exploiting and respecting the stable hover equilibrium in a single execution
        of the robotic system~\cite{abraham2019activelearning}.

        An LQR policy is used to keep the vehicle hovering while a local linear model (centered
        around the hover) is used for planning exploratory motions.
        The stochastic model of the quadcopter is of the
        form~\cite{gal2016improving}
        \begin{equation} \label{eq:stochastic_dynamics}
            dx \sim \mathcal{N}(f_\theta(x, u), \sigma_\theta(x))
        \end{equation}
        where $\mathcal{N}$ is a normal distribution with mean $f_\theta$, and variance
        $\sigma_\theta(x)$, and the change in the state is given by $dx\in \mathbb{R}^n$. Here,
        $f(x, u;\theta) = f_\theta(x, u) : \mathbb{R}^{n \times m} \to \mathbb{R}^n$ specifies a
        neural-network model of the dynamics and $\sigma(x; \theta) = \sigma_\theta(x)$ is a
        diagonal Gaussian $\sigma_\theta(x): \mathbb{R}^n \to \mathbb{R}^n$ which defines the
        uncertainty of the transition model at state $x$ all parameterized by the parameters
        $\theta$.

        We use $\text{KL-E}^3$ to enable the quadcopter to explore with respect to the variance of
        the model (that is, exploration in the state-space is generated based on how uncertain the
        transition model is at that state). In a similar manner as done in the previous subsection,
        we use a Boltzmann softmax function to create the distribution
        \begin{equation}
            p(s) = \frac{\exp(c \sigma_\theta(s))}{\int_{\mathcal{S}^v} \exp(c \sigma_\theta(\bar{s})) d\bar{s}}.
        \end{equation}
        A more complex target distribution can be built (see~\cite{wilson_TASE_param_id,
        abraham2019activelearning}), however; due to the over-parameterization of the neural-network model,
        using such methods would require significant computation.

        The stochastic model is optimized by maximizing the log likelihood of the model using the
        likelihood function
        \begin{equation}
            \mathcal{L} = \mathcal{N}( dx \mid f_\theta(x, u), \sigma_\theta(x))
        \end{equation}
        where updates to the parameters $\theta$ are defined through the gradient of the log
        likelihood function:
        \begin{equation}
            \theta \gets \theta + \alpha \sum_k \nabla_\theta \log \mathcal{L}.
        \end{equation}
        Here, a batch of $K$ measurements $\{ \hat{x}_k, d\hat{x}_k, u_k \}_{k=1}^K$ are uniformly
        sampled from the data buffer $\mathcal{D}$ where the subscript $k$ denotes the $k^\text{th}$
        time. A variation of Alg.~\ref{alg:KLE3} for model learning is provided in the
        Appendix in Algorithm~\ref{alg:KLE3_model_learning}.

        \begin{table}[]
            \centering
            \begin{tabular}{ccc}
            \hline
            Method          & Average Power Loss    & Average $\| u \|$    \\ \hline \hline
            $\text{\textbf{KL-E}}^3$ & 0.16 +- 0.0130   & 0.68 +- 0.0043 \\ \hline
            Inf. Max*    & 0.59 +- 0.0463 & 1.32 +- 0.3075 \\ \hline
            Normal 0.1*  & 1.41 +- 0.0121  & 0.72 +- 0.0016 \\ \hline
            OU 0.3       & 2.73 +- 0.0228   & 1.17 +- 0.0152 \\ \hline
            OU 0.1       & 0.97 +- 0.0096   & 0.73 +- 0.0033 \\ \hline
            OU 0.01*     & 0.10 +- 0.0007   & 0.67 +- 0.0002 \\ \hline
            Uniform 0.1* & 0.84 +- 0.0090   & 0.69 +- 0.0004 \\ \hline
            \end{tabular}
            \caption{
                Comparison of our method against various methods for state-space exploration using a
                quadcopter. Each method uses the same base stabilization policy which maintains
                hover height and is instantiated and run once for 1200 time steps. Data from a
                single execution is used to generate a neural network dynamics model. This is
                repeated $20$ times to estimate the performance of each method. Methods with (*)
                were unable to generate a dynamics model that completed the tracking objective.
                            }
            \label{fig:model_learning_metrics}
        \end{table}

        \begin{figure}[!]
            \centering
            \begin{subtable}[b]{0.24\textwidth}
                \centering
                \includegraphics[width=\linewidth]{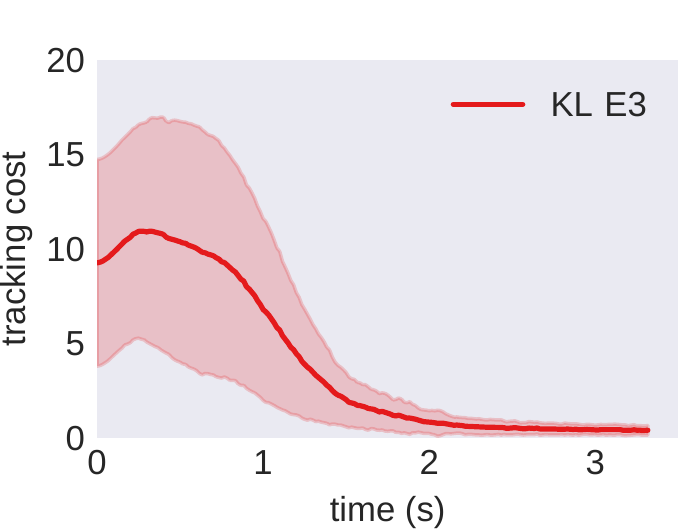}
                \label{fig:benchmarks}
            \end{subtable}
            \begin{subtable}[b]{0.24\textwidth}
                \centering
                \includegraphics[width=\linewidth]{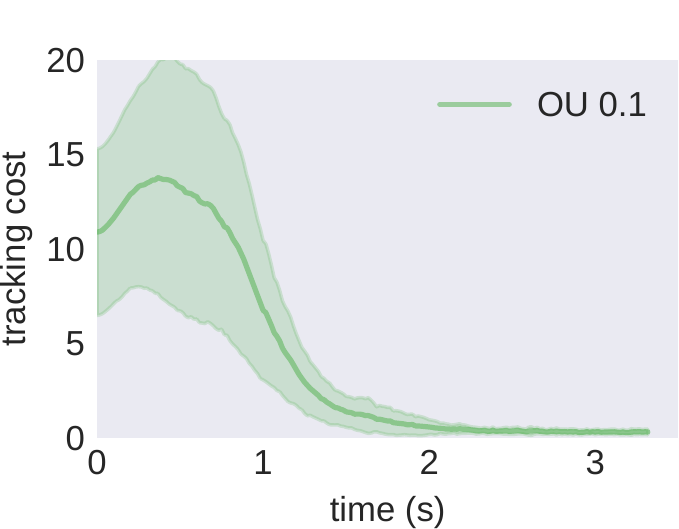}
                \label{fig:benchmarks}
            \end{subtable}
            \begin{subtable}[b]{0.24\textwidth}
                \centering
                \includegraphics[width=\linewidth]{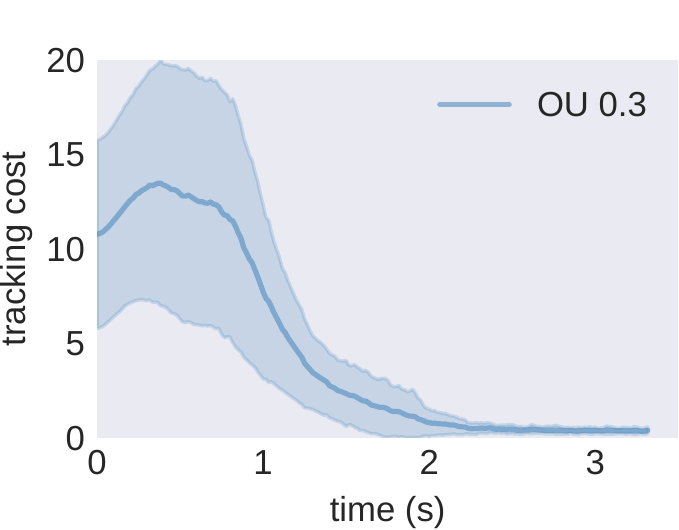}
                \label{fig:benchmarks}
            \end{subtable}
            \caption{
                Learned quadcopter model evaluations on a model-based tracking objective. Our method is
                able to generate a model that performs comparably to OU noise at $0.1$ and $0.3$ noise levels
                while using less energy through dynamic exploration.
            }
            \label{fig:model_learning_eval}
        \end{figure}

        We compare our method against time-correlated Ornstein-Uhlenbeck (OU)
        noise~\cite{uhlenbeck1930theory}, uniform and normally distributed random noise at different
        noise levels, and using a nonlinear dynamics variant of the information maximizing method
        in~\cite{abraham2019activelearning, wilson_TASE_param_id} which directly maximizes the
        variance of the model subject to the equilibrium policy. Each simulation is run using the
        LQR controller as a equilibrium policy for a single execution of the robot (no episodic
        resets) for $1200$ time steps. During this time, data is collected and stored in the buffer $\mathcal{D}$.
        Our method and the information maximizing method use the data in the stored buffer to update
        the variance $\sigma_\theta(x)$, guiding the exploration process. However, for
        evaluation of the transition model, we separately learn a model using the data that has been collected as a gauge
        for the utility of the collected data for each method. A stochastic model is learned by
        sampling a batch of $200$ measurements offline from the buffer using $2000$ gradient
        iterations. The model is evaluated for target tracking using stochastic model-based
        control~\cite{williams2016aggressive} over a set of uniformly randomly generating target
        locations $\mathcal{U}(-2,2)\in \mathbb{R}^3$.

        We first illustrate that our method is more energetically efficient compared to other
        methods in Table~\ref{fig:model_learning_metrics}. Here, energy is calculated using the
        resulting thrust of the quadcopter and we show the average commanded action $u$ over the
        execution of the quadrotor in time. Our method is shown to be more energetically efficient
        (due to the direct exploitation of the equilibrium policy and the ergodic exploration
        process). Furthermore, our method is able to generate measurements in the state-space that
        can learn a descriptive stochastic model of the dynamics for model-based tracking control
        (methods that could not learn a model for tracking control are indicated with a [*]). The
        resulting methods that could generate a model were comparable to our method (see
        Fig.~\ref{fig:model_learning_eval}), however; our method is able to directly target the
        regions of uncertainty (see Fig.~\ref{fig:collected_data}) through dynamic exploration
        allowing the quadcopter to use less energy (and more directed exploratory actions).

        \begin{figure}[th!]
            \centering
            \includegraphics[width=\linewidth]{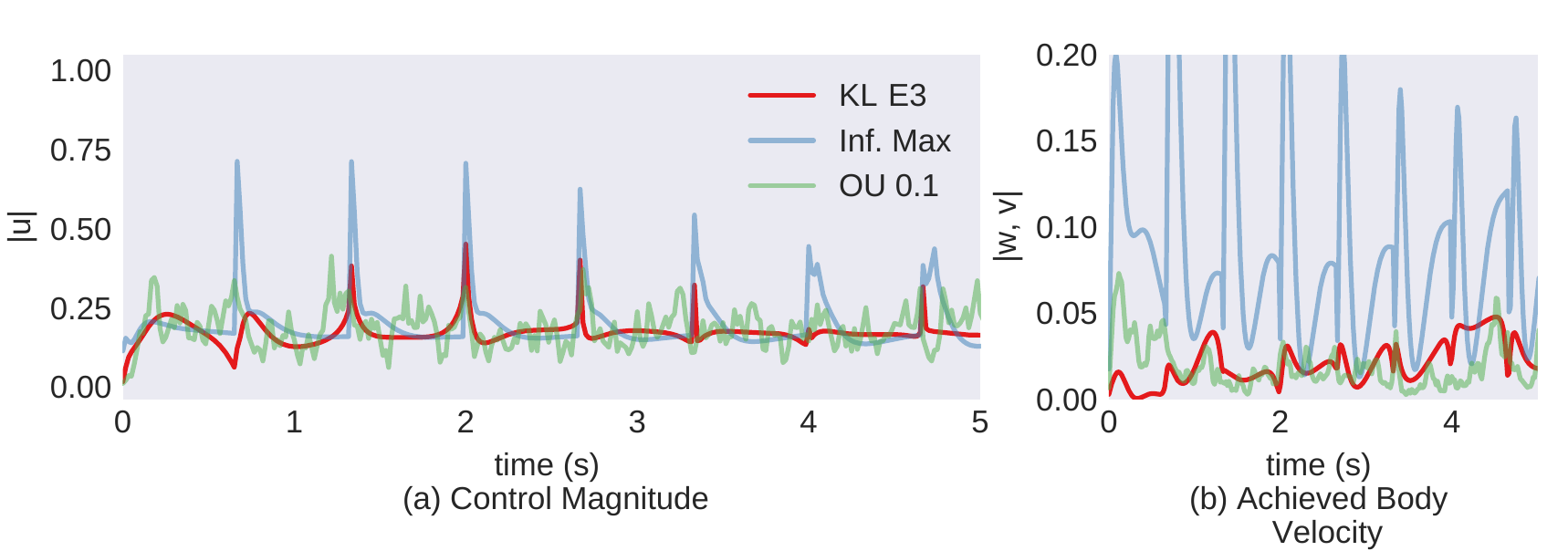}
              \caption{
              Control signal $\Vert u(t) \Vert$ and resulting body linear and angular velocities
              $\omega, v$ for the quadcopter system using our method, information maximization, and
              OU noise with $0.1$ maximum noise level. Our method generates smoother control signals
              while exploring the necessary regions of the state-space without destabilizing the
              system.
              }
             \label{fig:collected_data}
        \end{figure}

        Our last example illustrates how method can be used to aide exploration for off-policy robot
        skill learning methods by viewing the learned skill as an equilibrium policy.

    \subsection{Robot Skill Learning}

        In our last example, we explore $\text{KL-E}^3$ for improving robot skill learning (here we consider off-policy
        reinforcement learning). For all examples, we assume that the learned policy is the equilibrium
        policy and is simultaneously learned and utilized for safe exploration. As a result, we cannot confirm Lyaponov
        attractiveness, but assume that the learned policy will eventually yield the Lyaponov property.
        Thus, our goal is to show that we can consider a robot skill as being in equilibrium (using a feedback policy)
        where our method can explore within the vicinity of the robot skill in an intentional manner, improving the
        learning process.

        In many examples of robot skill learning, a common mode of failure is that the resulting
        learned skill is highly dependent on the quality of the distribution of data generated that
        is used for learning. Typically, these methods use the current iteration of the learned
        skill (which is often referred to as a policy) with added noise (or have a stochastic
        policy) to explore the action space. Often the added noise is insufficient towards gaining
        informative experience which improves the quality of the policy. Here, we show that our
        method can improve robot skill learning by generating dynamic coverage and exploration
        around the learned skill, reducing the likelihood of suboptimal solutions, and improving the
        efficiency of these methods.

        We use deep deterministic policy gradient (DDPG)~\cite{lillicrap2015continuous} as our
        choice of off-policy method.  Given a set of data, DDPG calculates a Q-function defined as
        \begin{equation}
            Q(x, u) = \mathbb{E}\left[ r(x, u) + \gamma Q(x', \mu(x')) \right]
        \end{equation}
        where $r(x, u) : \mathbb{R}^{n \times m} \to \mathbb{R}$ is a reward function, $x'$ is the
        next state subject to the control $u$, the expectation $\mathbb{E}$ is taken with respect to
        the states, $0 >\gamma > 1$ is known as a discounting factor~\cite{sutton2018reinforcement},
        and the function $Q(x, u) : \mathbb{R}^{n\times m} \to \mathbb{R}$ maps the utility of a
        state and how the action at that state will perform in the next state given the policy
        $\mu(x)$. DDPG simultaneously learns $Q(s)$ and a policy $\mu(x)$ by sampling from a set of
        collected states, actions, rewards, and their resulting next state. We refer the reader to
        the pseudo-code of DDPG in~\cite{lillicrap2015continuous}.

        \begin{figure}
            \centering
            \includegraphics[width=\linewidth]{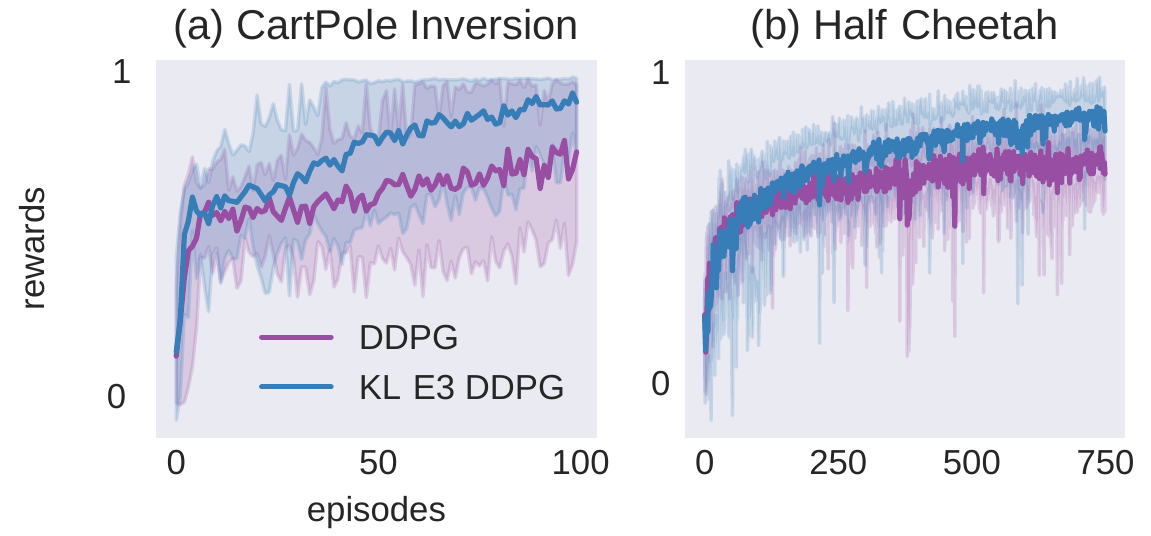}
              \caption{
              Comparison of $\text{KL-E}^3$ enhanced DDPG against DDPG using the cart pole inversion
              and the half cheetah running tasks. $\text{KL-E}^3$ provides a more informative
              distribution of data which assists the learning process, improves the overall performance, and
              achieves better performance faster for DDPG.
              }
             \label{fig:ddpg_comp}
        \end{figure}

        Our method uses the learned $Q(x,u)$ and $\mu(x)$ as the target distribution and the
        equilibrium policy respectively. We modify the Q function such that it becomes a
        distribution using a Boltzmann softmax
        \begin{equation}\label{eq:q_soft}
            p(s) = \frac{\exp(c Q(s))}{\int_{\mathcal{S}^v} \exp( c Q(\bar{s}))d\bar{s} }
        \end{equation}
        where $s \in \mathbb{R}^{n+m}$ includes both states and actions. This form of
        Equation~(\ref{eq:q_soft}) has been used previously for inverse reinforcement
        learning~\cite{biyik2018batch, brown2019risk}. Here, Eq.~(\ref{eq:q_soft}) is used as a
        guide for the ergodic exploration where our exploration is centered around the learned
        policy and the utility of the learned skill (Q-function). Since most reinforcement learning
        deals with large state-spaces, we use the approximation to the time-averaged statistics
        in~(\ref{eq:jensen_sigma}) to improve the computational efficiency of our algorithm. A
        parameterized dynamics model is built using the first $200$ points of each simulation (see
        Appendix for more detail) and updated as each trial continues. OU noise
        is used for exploration in the DDPG comparison with the same parameters shown
        in~\cite{lillicrap2015continuous}. We provide a pseudo-code of a $\text{KL-E}^3$ enhanced
        DDPG in the Appendix Alg.~\ref{alg:KLE3_DDPG}.

        Our method is tested on the cart pole inversion and the half-cheetah running task (see
        Figure~\ref{fig:ddpg_comp} for results). For both robotic systems, $\text{KL-E}^3$ is shown to
        improve the overall learning process, making learning a complex robot skill more sample
        efficient. Specifically, inverting the cart pole starts to occur within $50$
        episodes and the half cheetah begins generating running gaits within $250$ episodes of the
        half cheetah (each episode consists of $200$ time steps of each environment). In contrast,
        DDPG alone generates suboptimal running (as shown in
        (\url{https://sites.google.com/view/kle3/home})) and unstable cart inversion attempts. Because
        our method is able to explore within the vicinity of the learned skill in an intentional, ergodic
        manner, it is able to quickly learn skills and improve the overall quality of the exploration.

\section{Conclusion} \label{sec:conc}

    We present $\text{KL-E}^3$, a method which is shown to enable robots to actively generate
    informative data for various learning goals from equilibrium policies. Our method synthesizes
    ergodic coverage using a KL-divergence measure which generates data through exploiting dynamic
    movement proportional to the utility of the data. We show that hybrid systems theory can be used
    to synthesize a schedule of exploration actions that can incorporate learned policies and models
    in a systematic manner. Last, we present examples that illustrate the effectiveness of our
    method for collecting and generating data in an ergodic manner and provide theoretical analysis
    which bounds our method through Lyapunov attractiveness.

\section*{Acknowledgment}

    This material is based upon work supported by the National Science Foundation under Grants CNS 1837515.
    Any opinions, findings and conclusions or recommendations expressed in this material are those of the authors and
    do not necessarily reflect the views of the aforementioned institutions.

\appendices

\section{Proofs} \label{app:proofs}

        \noindent
        \subsection{Proof of Proposition 1}
        \begin{proof}
            Let us define the trajectory $x(t)$ switching from $\mu(x(\tau)) \to \mu_\star(\tau)$ for
            a duration of $\lambda$ as
            \begin{align}\label{eq:switched_traj}
                x(t) = x(t_0) & + \int_{t_0}^\tau f(x, \mu(x))dt
                + \int_{\tau}^{\tau + \lambda} f(x, \mu_\star) dt \\
                & + \int_{\tau + \lambda}^{t_f} f(x, \mu(x))dt \nonumber
            \end{align}
            where we drop the dependence on time for clarity. Taking the derivative of
            (\ref{eq:kl_objective}), using (\ref{eq:switched_traj}), with respect to the duration
            time $\lambda$ gives us the following expression:
            \begin{equation}\label{eq:kl_sensitivity1}
                \frac{\partial}{\partial \lambda} D_{\text{KL}}
                = -\sum_i  \frac{ p(s_i) }{q(s_i)} \int_{\tau + \lambda}^{t_f}
                \frac{\partial \psi}{ \partial x} ^\top \frac{\partial x}{\partial \lambda}
                dt.
            \end{equation}
            We obtain $\frac{\partial x}{\partial \lambda}$ by using Leibniz's rule to evaluate the derivative of
            (\ref{eq:switched_traj}) with respect to $\lambda$ at the integration boundary conditions to obtain the
            expression
            \begin{equation}\label{eq:xdlambda}
                \frac{\partial x (t) }{\partial \lambda} = (f_2 - f_1)
                + \int_{\tau + \lambda}^{t} \left( \frac{\partial f}{\partial x}
                    + \frac{\partial f}{\partial u} \frac{\partial \mu}{\partial x}\right) ^\top
                \frac{\partial x(s)}{\partial \lambda} ds
            \end{equation}
            where $s$ is a place holder variable for time, $f_2 = f(x(\tau + \lambda), \mu_\star(\tau + \lambda))$ and
            $f_1 = f(x(\tau + \lambda), \mu(x(\tau + \lambda))$. Noting that $\frac{\partial x}{\partial
            \lambda}$ is a repeated term under the integral, (\ref{eq:xdlambda}) is a linear convolution with
            initial condition $\frac{\partial x(\tau + \lambda)}{ \partial \lambda} = f_2 - f_1$.  As a result, we can
            rewrite (\ref{eq:xdlambda}) using a state-transition matrix~\cite{anderson2007optimal}
            \begin{equation*}
                \Phi(t, \tau + \lambda) = \exp \left(
                    \left(\frac{\partial f}{\partial x}
                    + \frac{\partial f}{\partial u} \frac{\partial \mu}{\partial x}\right)^\top
                    \left( t - \tau \right)
                \right)
            \end{equation*}
            with initial condition $f_2 - f_1$ as
            \begin{equation}\label{eq:state_trans}
                \frac{\partial x(t)}{\partial \lambda} = \Phi(t,\tau + \lambda) (f_2 - f_1).
            \end{equation}
            Using (\ref{eq:state_trans}) in (\ref{eq:kl_sensitivity1}) gives the following expression
            \begin{equation}\label{eq:kl_sensitivity2}
                \frac{\partial}{\partial \lambda} D_{\text{KL}}  =
                    -\sum_i  \frac{ p(s_i) }{q(s_i)} \int_{\tau + \lambda}^{t_f}
                    \frac{\partial \psi}{ \partial x}^\top \Phi(t, \tau + \lambda)
                    dt \left(f_2 - f_1 \right).
            \end{equation}
            Taking the limit as $\lambda \to 0$ we then set
            \begin{equation} \label{eq:adjoint_integral}
                \rho(\tau)^\top = -\sum_i  \frac{ p(s_i) }{q(s_i)} \int_{\tau}^{t_f}
                \frac{\partial \psi}{ \partial x}^\top \Phi(t, \tau  )
                dt
            \end{equation}
            in (\ref{eq:kl_sensitivity2}) which results in
            \begin{equation}
                \frac{\partial}{\partial \lambda} D_{\text{KL}} = \rho(\tau)^\top\left(f_2 - f_1 \right).
            \end{equation}
            Taking the derivative of (\ref{eq:adjoint_integral}) with respect to time $\tau$ yields the following:
            \begin{multline*}
                \frac{\partial}{\partial \tau} \rho(\tau)^\top =
                    \sum_i \frac{p(s_i)}{q(s_i)}\frac{\partial \psi}{\partial x}^\top \Phi(\tau, \tau) \\
                    - \sum_i \frac{p(s_i)}{q(s_i)}\int_\tau^{t_f}\frac{\partial \psi}{\partial x}^\top \frac{\partial }{\partial \tau}
                    \Phi(t, \tau) dt.
            \end{multline*}
            Since $\Phi(\tau, \tau) = 1$, and $$\frac{\partial}{\partial \tau} \Phi(t, \tau) = -\Phi(t, \tau)\left( \frac{\partial
            f}{\partial x} + \frac{\partial f}{\partial u} \frac{\partial \mu}{\partial x}\right) ,$$
            we can show that
            \begin{multline*}
                \frac{\partial}{\partial \tau} \rho(\tau)^\top =
                    \sum_i \frac{p(s_i)}{q(s_i)}\frac{\partial \psi}{\partial x}^\top \\
                    -
                    \underbrace{
                        \left( -\sum_i \frac{p(s_i)}{q(s_i)}\int_\tau^{t_f}\frac{\partial \psi}{\partial x}^\top\Phi(t, \tau) dt \right)
                        }_{=\rho(\tau)^\top} \left( \frac{\partial f}{\partial x} + \frac{\partial f}{\partial u} \frac{\partial \mu}{\partial x}\right) .
            \end{multline*}
            Taking the transpose, we can show that $\rho(t)$ can be solved
            backwards in time with the differential equation
            \begin{equation}
                \dot{\rho}(t) =
                    \sum_i \frac{p(s_i)}{q(s_i)} \frac{\partial \psi}{\partial x}
                     -
                     \left(
                    \frac{\partial f}{\partial x} + \frac{\partial f}{\partial u} \frac{\partial \mu}{\partial x}
                 \right)^\top \rho(t)
            \end{equation}
            with final condition $\rho(t_f) = \mathbf{0}$.
        \end{proof}

        \noindent
        \subsection{Proof of Theorem 1}
        \begin{proof}
            Writing the integral form of the Lyapunov function switching between $\mu(x(t))$ and
            $\mu_\star(t)$ at time $\tau$ for a duration of time $\lambda$ starting at $x(0)$ can be
            written as
            \begin{align}\label{eq:lyap_switch}
                V(x^\tau_\lambda(t)) = V(x(0)) + &\int_{0}^{\tau} \nabla V \cdot f(x(s), \mu(x(s)) ) ds \nonumber \\
                + & \int_{\tau}^{\tau+\lambda} \nabla V \cdot f(x(s), \mu_\star(s)) ds \nonumber \\
                + & \int_{\tau+\lambda}^{t}  \nabla V \cdot f(x(s), \mu(x(s)) ) ds
            \end{align}
            where $s$ is a place holder for time. Expanding $f(x,u)$ to $g(x) + h(x)u$ and using
            (\ref{eq:explr_actions}) we can show the following identity:
            \begin{align}\label{eq:lyap_chain}
                \nabla V \cdot f(x, \mu_\star) & = \nabla V \cdot g(x) + \nabla V \cdot h(x) \mu_\star \nonumber \\
                & = \nabla V \cdot g(x) + \nabla V \cdot h(x) \mu(x) \nonumber \\
                                & \qquad \qquad - \nabla V \cdot h(x) \mathbf{R}^{-1} h(x)^\top \rho \nonumber \\
                & = \nabla V \cdot f(x, \mu(x)) - \nabla V \cdot h(x) \mathbf{R}^{-1}h(x)^\top \rho
            \end{align}
            Using (\ref{eq:lyap_chain}) in (\ref{eq:lyap_switch}), we can show that
            \begin{align}\label{eq:lyap_cool_form}
                V(x^\tau_\lambda(t)) &= V(x(0))
                        + \int_{0}^{t} \nabla V \cdot f(x(s), \mu(x(s)) ) ds \nonumber \\
                        & \quad \quad - \int_{\tau}^{\tau + \lambda} \nabla V \cdot h(x(s)) \mathbf{R}^{-1} h(x(s))^\top \rho(s) ds \nonumber\\
                 &= V(x(t)) \nonumber\\
                 & \quad \quad- \int_{\tau}^{\tau + \lambda} \nabla V \cdot h(x(s)) \mathbf{R}^{-1} h(x(s))^\top \rho(s) ds
            \end{align}
            where  $x(t)$ is given by (\ref{eq:trajectory}).

            Letting the largest value of $\nabla V \cdot h(x(t)) \mathbf{R}^{-1} h(x(t))^\top \rho(t)$ be given by
            \begin{equation}
                \beta = \sup_{s \in \left[ \tau, \tau + \lambda \right]} - \nabla V \cdot h(x(s)) \mathbf{R}^{-1}h(x(s))^\top \rho(s)>0,
            \end{equation}
            we can approximate (\ref{eq:lyap_cool_form}) as
            \begin{align*}
                V(x^\tau_\lambda (t)) &= V(x(t))- \int_{\tau}^{\tau + \lambda} \nabla V \cdot h(x(s))\mathbf{R}^{-1} h(x(s))^\top \rho(s) ds\\
                & \le V(x(t)) + \beta \lambda.
            \end{align*}
            Subtracting both side by $V(x(t))$ gives the upper bound
            \begin{equation}
                V(x^\tau_\lambda(t)) - V(x(t)) \le \beta \lambda
            \end{equation}
            which quantifies how much (\ref{eq:explr_actions}) deviates from the equilibrium conditions in (\ref{eq:lyap_cond}).
        \end{proof}

\section{Algorithmic Details} \label{app:bayes_opt}

    This appendix provides additional details for each learning goal presented in
    Section~\ref{sec:ex}. This includes pseudo-code for each method and parameters to implement our
    examples (see Table.~\ref{tab:param}). We provide videos of each example and demo code in
    (\url{https://sites.google.com/view/kle3/home}).

    \begin{table*}
        \centering
        \begin{tabular}{@{}cccccccl@{}}
        \toprule
        Example                                                                                 & $t_H$                      & $\lambda$                                                              & $f(x,u)$                                                                                                                      & $\mu(x)$                                                                             & $\mathbf{R}$                                 & $\Sigma$                             & $N$ samples \\ \midrule
        \begin{tabular}[c]{@{}c@{}}Cart Double Pendulum\\ Bayes. Opt.\end{tabular}              & 0.2 s                      & \begin{tabular}[c]{@{}c@{}}line search \\ $\lambda < t_H$\end{tabular} & \begin{tabular}[c]{@{}c@{}}local linear model \\ at inverted pose\end{tabular}                                                & \begin{tabular}[c]{@{}c@{}}LQR stabilizing\\ policy\end{tabular}                     & 0.1                                          & 0.1                                  & 20          \\
        \begin{tabular}[c]{@{}c@{}}Quadcopter Model\\ Learning\end{tabular}                     & 0.6 s                      & $\lambda = t_H$                                                        & \begin{tabular}[c]{@{}c@{}}local linear model \\ at hoverheight\end{tabular}                                                  & LQR hovering policy                                                                  & $0.5 \mathbf{I}$ & $0.1 \mathbf{I} $                    & 100         \\
        DDPG Cart pole swingup                                                                  & 0.1 s                      & $\lambda = 0.02$ s                                                     & \begin{tabular}[c]{@{}c@{}}neural-net model\\ $\dot{x} = f(x,u; \theta)$\\ learned from data\end{tabular}                     & Learned swingup skill                                                                & $0.01 * 0.99 ^ t$                            & $0.1 \mathbf{I}$                     & 20          \\
        \multicolumn{1}{l}{\begin{tabular}[c]{@{}l@{}}DDPG Half Cheetah\\ running\end{tabular}} & \multicolumn{1}{l}{0.03 s} & \multicolumn{1}{l}{$\lambda = 0.0165 s$}                               & \multicolumn{1}{l}{\begin{tabular}[c]{@{}l@{}}neural-net model\\ $\dot{x} = f(x,u; \theta)$\\ learned from data\end{tabular}} & \multicolumn{1}{l}{\begin{tabular}[c]{@{}l@{}}Learned running \\ skill\end{tabular}} & \multicolumn{1}{l}{$0.01 * 0.99^t$}          & \multicolumn{1}{l}{$0.1 \mathbf{I}$} & 50          \\ \bottomrule
    \end{tabular} \caption{
        Parameters used for each method presented in Section~\ref{sec:ex} . $\mathbf{I}$ indicates an
        identity matrix of size $m\times m$, $n \times n$
        for the quadcopter model learning, and $v \times v$ where $v=n+m$ used in both
        DDPG examples. Neural net model parameters $\theta$ are learned by minimizing the error
        $\Vert \dot{x} - f(x,u ;\theta) \Vert^2$ over a subset of $K$ data points where $\dot{x}
        \approx (x(t_1) - x(t_0)) / dt$ and $dt$ is the time step of the system. A time-decaying
        $\mathbf{R}$ is used for both DDPG examples so that the largest exploring occurs earlier on
        in the episode to better assist the skill learning.
    } \label{tab:param}
    \end{table*}

        \begin{algorithm}[h]
            \caption{Bayesian Optimization}
            \begin{algorithmic}[1]
                \State \textbf{init:} Gaussian prior on objective $\phi$, data set $\mathcal{D}$, $i=0$
                \While{task not done}
                    \State update posterior distribution on $\phi$ using $\mathcal{D}$
                    \State build acquisition function using current posterior on $\phi$
                    \State active learner finds the maximum $x_i$ of the acquisition function
                    \State active learner samples $y_i = \phi(x_i)$
                    \State $i \gets i + 1$
                    \State set $x(t_i) = \hat{x}(t_i)$
                \EndWhile
            \State return max $y_i$ of $\phi(x)$ and argmax $x_i$
            \end{algorithmic}
            \label{alg:bayes_opt}
        \end{algorithm}

        \begin{algorithm}[H]
            \caption{$\text{KL-E}^3$ for Bayesian Optimization}
            \begin{algorithmic}[1]
                \State \textbf{init:} see Alg.~\ref{alg:KLE3} and Alg.~\ref{alg:bayes_opt}
                \While{task not done}
                    \State set $x(t_i) = \hat{x}(t_i)$
                    \State $\triangleright$ simulation loop (see Alg.~\ref{alg:KLE3} lines 4-18)
                    \State get $\delta \mu(t)$ from simulation
                    \State $\triangleright$ apply to real robot (see Alg.~\ref{alg:KLE3} lines 20-28)
                    \State chose $\tau \in [t_i, t_i+t_H]$ and $\lambda \le t_H$
                    \State $\triangleright$
                    \For{$t \in [t_i, t_{i+1}] $}
                        \If{ $t \in [\tau, \tau + \lambda]$}
                            \State apply $\mu_\star(t) = \delta\mu_\star(t) + \mu(\hat{x}(t))$
                        \Else
                            \State apply $\mu(x(t))$
                        \EndIf
                        \If{time to sample}
                            \State measure true state $\hat{x}(t)$ and $y(t)= \phi(\hat{x}(t))$
                            \State append to data set $\mathcal{D} \gets \{ \hat{x}(t), y(t)\}$
                        \EndIf
                    \EndFor
                    \State update posterior on $\phi$ given $\mathcal{D}$
                    \State update $p(s)$ from posterior
                    \State $i \gets i + 1$
                \EndWhile
            \end{algorithmic}
            \label{alg:KLE3_bayes_opt}
        \end{algorithm}

\vfill
    \begin{algorithm}[H]
        \caption{$\text{KL-E}^3$ for Model Learning}
        \begin{algorithmic}[1]
            \State \textbf{init:} see Alg.~\ref{alg:KLE3}, generate initial parameter $\theta^0$ for model (\ref{eq:stochastic_dynamics})
            \While{task not done}
                \State set $x(t_i) = \hat{x}(t_i)$
                \State $\triangleright$ simulation loop (see Alg.~\ref{alg:KLE3} lines 4-18)
                \State get $\delta \mu(t)$ from simulation
                \State $\triangleright$ apply to real robot
                \State chose $\tau \in [t_i, t_i+t_H]$ and $\lambda \le t_H$
                \For{$t \in [t_i, t_{i+1}] $}
                    \If{ $t \in [\tau, \tau + \lambda]$}
                        \State apply $\mu_\star(t) = \delta\mu_\star(t) + \mu(\hat{x}(t))$
                    \Else
                        \State apply $\mu(x(t))$
                    \EndIf
                    \If{time to sample}
                        \State measure state $\hat{x}(t)$, change in state $d \hat{x}(t)$, and applied control $u(t)$
                        \State append to data set $\mathcal{D} \gets \{ \hat{x}(t), d\hat{x}(t), u(t)\}$
                    \EndIf
                \EndFor
                \State sample $K$ batch $\{\hat{x}_k, d\hat{x}_k, u_k \}_{k=1}^K$
                \State $\theta^{i+1} \gets \theta^{i} + \sum_k \nabla_\theta \log \mathcal{L}$ given batch
                \State update $p(s)$ from $\sigma_{\theta^{i+1}}$
                \State $i \gets i + 1$
            \EndWhile
        \end{algorithmic}
        \label{alg:KLE3_model_learning}
    \end{algorithm}

    \begin{algorithm}[H]
        \caption{$\text{KL-E}^3$ enhanced DDPG}
        \begin{algorithmic}[1]
            \State \textbf{init:} see Alg.~\ref{alg:KLE3}, ~\cite{lillicrap2015continuous},
            \While{task not done}
                \State set $x(t_i) = \hat{x}(t_i)$
                \State $\triangleright$ simulation loop (see Alg.~\ref{alg:KLE3} lines 4-18)
                \State get $\delta \mu(t)$ from simulation and apply to real robot
                \State chose $\tau \in [t_i, t_i+t_H]$ and $\lambda \le t_H$
                \For{$t \in [t_i, t_{i+1}] $}
                    \If{ $t \in [\tau, \tau + \lambda]$}
                        \State apply $\mu_\star(t) = \delta\mu_\star(t) + \mu(\hat{x}(t))$
                    \Else
                        \State apply $\mu(x(t))$
                    \EndIf
                    \If{time to sample}
                        \State measure state $\hat{x}(t)$, next state $\hat{x}'(t)$, applied control $u(t)$, reward $r(t)$
                        \State append $\mathcal{D} \gets \{ \hat{x}(t), \hat{x}'(t), u(t), r(t)\}$
                    \EndIf
                \EndFor
                \If{time to update and buffer is large enough}
                    \State update $Q$, $\mu$ from ~\cite{lillicrap2015continuous}
                    \State update $p(s)$ using $Q$ (\ref{eq:q_soft})
                \EndIf
                \State $i \gets i + 1$
            \EndWhile
        \end{algorithmic}
        \label{alg:KLE3_DDPG}
    \end{algorithm}

\newpage

\bibliographystyle{IEEEtran/IEEEtran.bst}
\bibliography{references}
\newpage

\begin{IEEEbiography}[{\includegraphics[width=1in,height=1.25in,clip,keepaspectratio]{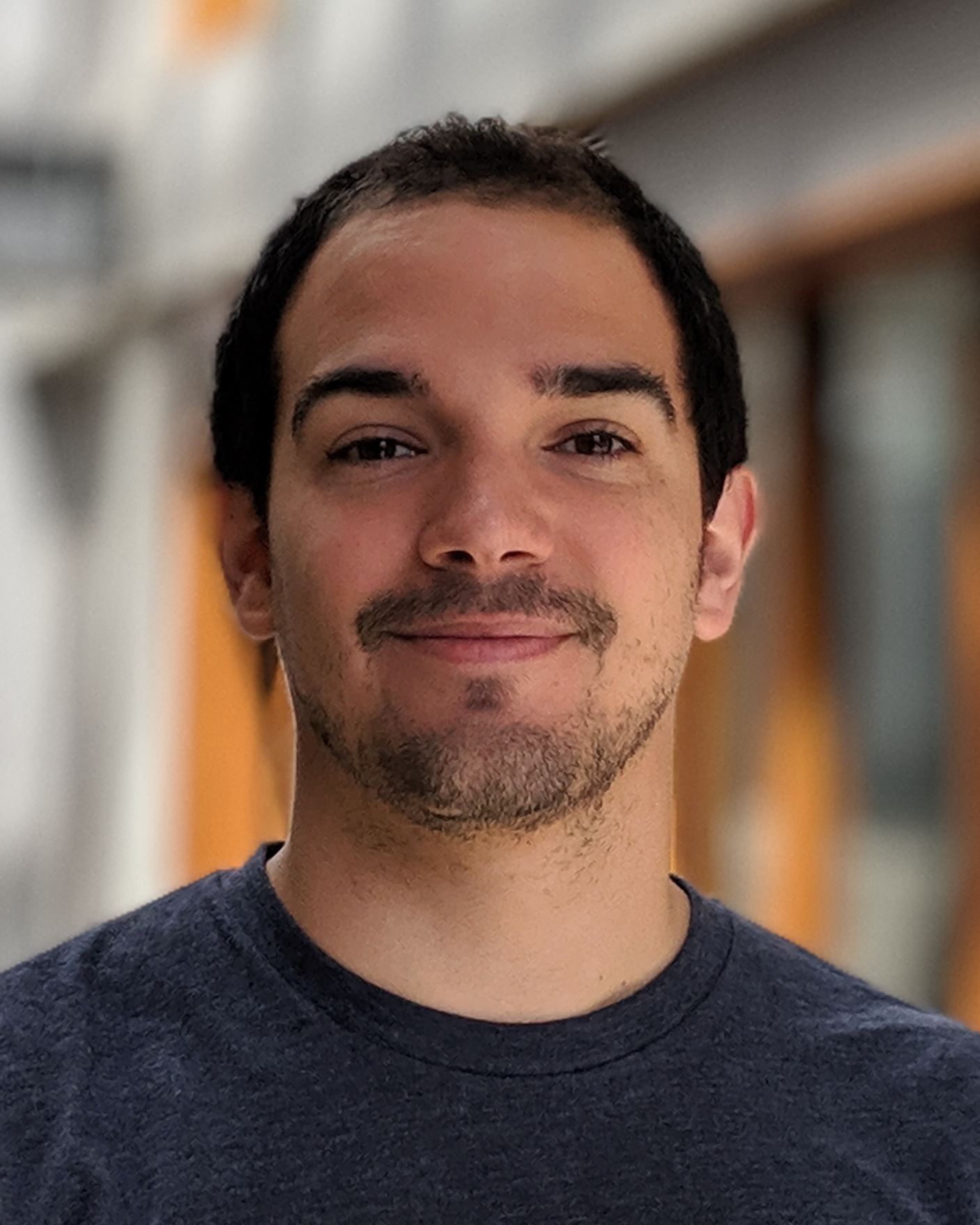}}]{Ian Abraham}
Ian Abraham received the B.S. degree in Mechanical and Aerospace Engineering from Rutgers University and the M.S. degree
in Mechanical Engineering from Northwestern University. He is currently a Ph.D. Candidate at the Center for Robotics and
Biosystems at Northwestern University. His Ph.D. work focuses on developing formal methods for robot sensing and runtime
active learning. He is also the recipient of the 2019 King-Sun Fu IEEE Transactions on Robotics Best Paper award.
\end{IEEEbiography}

\vspace{-60mm}

\begin{IEEEbiography}[{\includegraphics[width=1in,height=1.25in,clip,keepaspectratio]{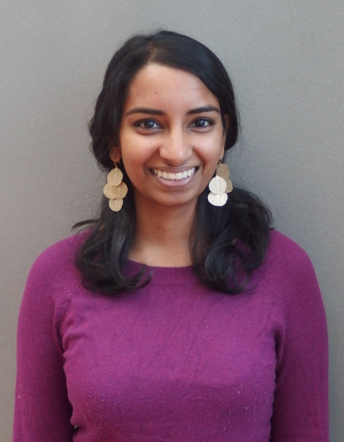}}]{Ahalya Prabhakar}
Ahalya Prabhakar received the B.S. degree in mechanical engineering from California Institute of Technology, Pasadena,
CA, USA, in 2013, and the M.S. degree in mechanical engineering from Northwestern University, Evanston, IL, USA, in
2016. She is a Ph.D. candidate at the Center for Robotics and Biosystems at Northwestern University. Her work focuses on
developing compressible representations through active exploration for complex task performance and efficient robot
learning.
\end{IEEEbiography}

\vspace{-60mm}

\begin{IEEEbiography}[{\includegraphics[width=1in,height=1.25in,clip,keepaspectratio]{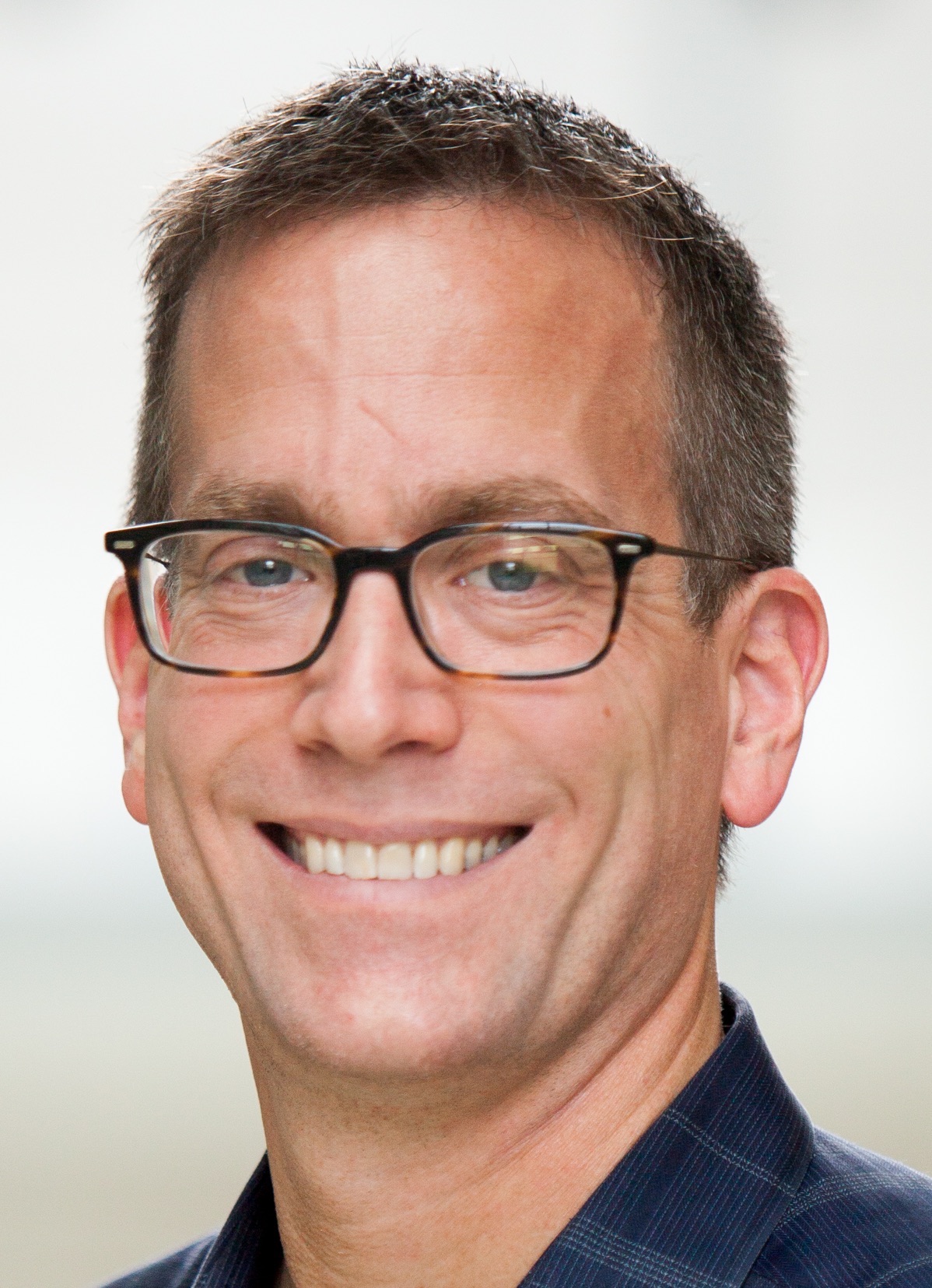}}]{Todd D. Murphey}
Todd D. Murphey received his B.S. degree in mathematics from the University of Arizona and the Ph.D. degree in Control
and Dynamical Systems from the California Institute of Technology. He is a Professor of Mechanical Engineering at
Northwestern University. His laboratory is part of the Neuroscience and Robotics Laboratory, and his research interests
include robotics, control, computational methods for biomechanical systems, and computational neuroscience. Honors
include the National Science Foundation CAREER award in 2006, membership in the 2014-2015 DARPA/IDA Defense Science
Study Group, and Northwestern’s Professorship of Teaching Excellence. He was a Senior Editor of the IEEE Transactions on
Robotics.
\end{IEEEbiography}
\end{document}